\documentclass[letterpaper]{article} 
\usepackage{aaai24}  
\usepackage{times}  
\usepackage{helvet}  
\usepackage{courier}  
\usepackage[hyphens]{url}  
\usepackage{graphicx} 
\usepackage{natbib}  
\usepackage{caption} 
\usepackage{algorithm}
\usepackage{algorithmic}

\usepackage{newfloat}
\usepackage{listings}
\DeclareCaptionStyle{ruled}{labelfont=normalfont,labelsep=colon,strut=off} 
\floatstyle{ruled}
\newfloat{listing}{tb}{lst}{}
\floatname{listing}{Listing}
\usepackage{kotex}
\usepackage{bm}
\usepackage{amsfonts}
\usepackage{xspace}
\usepackage{booktabs}
\usepackage{multirow}
\usepackage{makecell}
\usepackage{xcolor}
\usepackage{nicefrac}
\usepackage{amsmath}
\usepackage{arydshln}
\usepackage{graphicx}
\usepackage{pifont}
\usepackage{dsfont}
\usepackage{pifont}
\usepackage{scalerel}
\usepackage{enumitem}
\usepackage{subcaption}
\usepackage{cellspace}
\usepackage{makecell}
\usepackage[hang,flushmargin]{footmisc} 
\usepackage{amsthm}
\usepackage{mathtools}

\newtheorem{lemma}{Lemma}

\newcommand{\repo}{\url{https://github.com/wooner49/T-spear-shield}}

\newcommand{\vece}{\boldsymbol{e}}

\newcommand{\vech}{\boldsymbol{h}}

\newcommand{\vecm}{\boldsymbol{m}}

\newcommand{\vecs}{\boldsymbol{s}}

\newcommand{\vecv}{\boldsymbol{v}}

\newcommand{\matA}{\mathbf{A}}

\newcommand{\matE}{\mathbf{E}}

\newcommand{\calE}{\mathcal{E}}

\newcommand{\calG}{\mathcal{G}}

\newcommand{\calL}{\mathcal{L}}

\newcommand{\calN}{\mathcal{N}}

\newcommand{\calV}{\mathcal{V}}

\newcommand{\edgei}{e_{i}=(u_{i},v_{i},t_{i})}

\newcommand{\tb}[1]{\textbf{#1}\xspace}
\newcommand{\ul}[1]{\underline{#1}\xspace}

\newcommand{\TGN}{\textsc{TGN}\xspace}
\newcommand{\JODIE}{\textsc{JODIE}\xspace}
\newcommand{\TGAT}{\textsc{TGAT}\xspace}
\newcommand{\DySAT}{\textsc{DySAT}\xspace}

\newcommand{\random}{\textsc{Random}\xspace}
\newcommand{\pa}{\textsc{PA}\xspace}
\newcommand{\jaccard}{\textsc{Jaccard}\xspace}
\newcommand{\degree}{\textsc{Struct-D}\xspace}
\newcommand{\pagerank}{\textsc{Struct-PR}\xspace}
\newcommand{\attack}{\textsc{T-Spear}\xspace}

\newcommand{\tgnsvd}{\textsc{TGN-SVD}\xspace}
\newcommand{\tgncosine}{\textsc{TGN-Cosine}\xspace}
\newcommand{\defensefilter}{\textsc{T-Shield-F}\xspace}

\newcommand{\defense}{\textsc{T-Shield}\xspace}

\DeclareMathOperator*{\argmin}{arg\,min}

\title{Spear and Shield: Adversarial Attacks and Defense Methods for Model-Based Link Prediction on Continuous-Time Dynamic Graphs}

\author{Dongjin Lee\textsuperscript{\rm 1},
Juho Lee\textsuperscript{\rm 2}, Kijung Shin\textsuperscript{\rm 1,2}}
\affiliations{\textsuperscript{\rm 1}School of Electrical Engineering, Daejeon, South Korea\\
        \textsuperscript{\rm 2}Kim Jaechul Graduate School of AI, KAIST, Seoul, South Korea\\ 
	\{dongjin.lee, juholee, kijungs\}@kaist.ac.kr
}

\begin{document}
\maketitle

\begin{abstract}
Real-world graphs are dynamic, constantly evolving with new interactions, such as financial transactions in financial networks. 
Temporal Graph Neural Networks (TGNNs) have been developed to effectively capture the evolving patterns in dynamic graphs. 
While these models have demonstrated their superiority, being widely adopted in various important fields, their vulnerabilities against adversarial attacks remain largely unexplored.
In this paper, we propose \attack, a simple and effective adversarial attack method for link prediction on continuous-time dynamic graphs, focusing on investigating the vulnerabilities of TGNNs.
Specifically, before the training procedure of a victim model, which is a TGNN for link prediction, we inject edge perturbations to the data that are unnoticeable in terms of the four constraints we propose, and yet effective enough to cause malfunction of the victim model. 
Moreover, we propose a robust training approach \defense to mitigate the impact of adversarial attacks.
By using edge filtering and enforcing temporal smoothness to node embeddings, we enhance the robustness of the victim model.
Our experimental study shows that \attack significantly degrades the victim model's performance on link prediction tasks, and even more, our attacks are transferable to other TGNNs, which differ from the victim model assumed by the attacker.
Moreover, we demonstrate that \defense effectively filters out adversarial edges and exhibits robustness against adversarial attacks, surpassing the link prediction performance of the na\"ive TGNN by up to 11.2\% under \attack.
The code and datasets are available at \repo.

\end{abstract}

\section{Introduction}\label{sec:intro}
Graph Neural Networks (GNNs) have shown impressive capabilities across various tasks~\citep{sun2022adversarial,jin2021adversarial}.
However, their susceptibility to adversarial perturbations~\citep{bojchevski2019adversarial,dai2018adversarial,zugner2018adversarial,xu2019topology} has raised concerns about their safe use in critical applications, including fraud detection~\citep{hooi2016fraudar,zeager2017adversarial} and malware detection~\citep{hou2019alphacyber}. 
Despite substantial progress made in developing attack and defense techniques~\citep{sun2022adversarial,jin2021adversarial}, the focus has primarily been on GNNs designed for static graphs.

In the real world, graphs are dynamic, constantly changing over time. 
For example, in financial networks, transactions between accounts occur continuously over time, leading to ever-changing graphs.
Dynamic graphs fall into two types: Discrete-Time Dynamic Graphs (DTDGs), denoted as a sequence of static graph snapshots, and Continuous-Time Dynamic Graphs (CTDGs), represented as a sequence of interactions over continuous time.
Due to the capability of CTDGs to capture more intricate and fine-grained temporal patterns in dynamic graphs, CTDGs have recently garnered increasing attention, resulting in the development of numerous machine learning models on CTDGs for time-critical applications~\citep{kazemi2020representation,thomas2022graph}.

To effectively capture the evolving patterns in CTDGs, Temporal Graph Neural Networks (TGNNs)~\citep{rossi2020temporal,kumar2019predicting,xu2020inductive,wang2021apan,zhou2022tgl} have been developed.
TGNNs jointly learn the temporal, structural, and contextual relationships present in CTDGs by encoding graph information into time-aware node embeddings.
By incorporating temporal information, TGNNs have demonstrated superior performance over static GNNs in various tasks like link prediction~\citep{cong2023we,wang2021inductive} and dynamic node classification~\citep{rossi2020temporal}.

Nonetheless, in contrast to the actively researched adversarial attacks on static GNNs, the vulnerabilities of TGNNs to adversarial attacks remain underexplored, yet the significance of conducting such research is undeniable.
For instance, in a financial attack~\citep{fursov2021adversarial,zeager2017adversarial} scenario, an adversary may inject adversarial transactions into the transaction graph, perturbing the timing and content of transactions to mislead the model running on the graph. 
This manipulation could lead to the model falsely predicting fraudulent transactions as legitimate, resulting in inaccurate risk predictions and potential financial losses.

In this context, we explore the weakness of TGNNs by introducing \attack, a simple and effective adversarial attack method for link prediction on CTDGs.
From an adversarial perspective, designing effective attacks on CTDGs poses unique challenges compared to attacks on static graphs.
First, in addition to the perturbation injected into graphs, attacks should carefully be timed.
Second, attacks should stay stealthy and not deviate significantly from the usual evolving patterns of interactions in the original graph.
To address these challenges, we impose four constraints on perturbations to ensure the attack remains unnoticeable: in terms of the (C1) perturbation budget, (C2) distribution of time, (C3) endpoints of adversarial edges, and (C4) number of perturbations per node.
While satisfying the above constraints, we generate edge perturbations by selecting edges that are unlikely to be formed when considering the evolution of dynamic graphs, and then inject them into the data before the training process of a victim model (i.e., a poisoning attack).
Benchmarking on link prediction tasks, \attack outperforms baselines, all within the same perturbation budget.
Moreover, \attack remains effective even when the victim model differs from the model assumed by the attacker, indicating that \attack is transferrable. 


In terms of the defense against adversarial attacks, the techniques available for static GNNs~\citep{mujkanovic2022defenses} are not adequate, partly because they do not consider the dynamic nature of the graphs. For instance, in CTDGs, a normal edge may turn into an adversarial edge over time, but a defense algorithm for static GNNs cannot adapt to this change.
To mitigate the impact of adversarial attacks on CTDGs, we propose \defense, a robust training method for TGNNs.
By employing the edge filtering technique that considers graphs' dynamics, \defense eliminates potential adversarial edges from the corrupted dynamic graphs without any prior knowledge of the attacks.
In addition, we enforce temporal smoothness on node embeddings to prevent abrupt changes, further enhancing the models' robustness. 

In summary, our contributions are as follows:
\begin{itemize}[leftmargin=*,noitemsep,topsep=0pt,parsep=0pt,partopsep=0pt]
    \item To the best of our knowledge, we are the first to formulate adversarial attacks on CTDGs under realistic constraints regarding unnoticeability.
    \item We propose \attack, a simple and effective poisoning attack method on CTDGs, and demonstrate its superiority over 5 baselines across 4 different datasets and 4 victim models on link prediction tasks. 
    \item We propose \defense, a robust training method for TGNNs that shows robustness against adversarial attacks, surpassing the na\"ive TGNN by up to 11.2\% on link prediction tasks.
\end{itemize}

\section{Related Work}\label{sec:related}

\subsubsection{Temporal Graph Neural Networks (TGNNs).}
TGNNs are categorized into two types: discrete-time TGNNs and continuous-time TGNNs.
Discrete-time TGNNs operate on sequences of static graph snapshots that are evenly sampled.
Many discrete-time TGNNs integrate spatial and temporal information by combining GNNs with sequence models~\citep{sankar2018dynamic,seo2018structured,pareja2020evolvegcn,kazemi2020representation}.
Especially, DySAT~\citep{sankar2018dynamic} leverages structural and temporal self-attention on DTDGs.

The continuous-time TGNNs to which our work relates directly learn the time-aware node embeddings on CTDGs, allowing them to capture finer temporal patterns compared to discrete-time TGNNs. 
Specifically, JODIE~\citep{kumar2019predicting} employs two Recurrent Neural Network (RNN) modules to sequentially update the source and destination node embeddings as new edges arrive.
TGAT~\citep{xu2020inductive} employs an attention-based message-passing to aggregate messages from a historical neighborhood.
To alleviate the expensive neighborhood aggregation of TGAT, TGN~\citep{rossi2020temporal} employs an RNN memory module to encode the history of each node.
Our proposed method mainly focuses on targeting the above three representative continuous-time TGNNs.

In addition to these, various studies~\citep{jin2022neural,wang2021apan,zhou2022tgl,cong2023we} have been conducted on continuous-time TGNNs.
Most research on continuous-time TGNNs has demonstrated their advancement using link prediction as a downstream task, so we mainly investigate link prediction as our target task to study the effect of adversarial attacks.

\subsubsection{Adversarial Attacks on Graphs.}
Many studies have delved into adversarial attacks on static graphs~\citep{sun2022adversarial,jin2021adversarial}.
Moreover, graph adversarial attacks have been explored from diverse angles, such as perturbation types~\citep{liu2019unified} and unnoticeability~\citep{zugner2018adversarial,dai2023unnoticeable}.

Some studies consider adversarial attacks on DTDGs, focusing on degrading the link prediction performance.
Specifically, TGA~\citep{chen2021time} is a white-box evasion attack that greedily selects edge perturbations across graph snapshots.
TD-PGD~\citep{sharma2023temporal} is a Projected Gradient Descent (PGD) based attack under temporal dynamics-aware perturbation constraint.
However, these methods are only applicable to DTDGs and cannot generate edge perturbations with continuous timestamps.
Despite diverse studies, adversarial attacks on CTDGs are insufficiently studied.

\subsubsection{Adversarial Defenses on Graphs.}
A bunch of work has proposed adversarial defense methods for static graphs~\citep{mujkanovic2022defenses}.
Especially, GCN-SVD~\citep{entezari2020all} substitutes the adjacency matrix with its low-rank approximation before plugging it into a regular GNN. 
ProGNN~\citep{jin2020graph} alternately optimizes the parameters of the GNN and the graph structure.
GNNGuard~\citep{zhang2020gnnguard} filters edges in the message passing stage using cosine-similarity.
However, these methods cannot be directly applied to CTDGs, and repeatedly applying them whenever new interactions occur is computationally prohibitive, and further fail to consider the dynamics in graphs.

In contrast, the literature on adversarial defense for dynamic graphs is limited.
AdaNet~\citep{li2022robust} proposes an adaptive neighbor selection technique for reliable message propagation on CTDGs.
However, the method demonstrates its robustness against small amounts of contextual noise such as outdated links and node feature noise, rather than structural perturbations that are our main focus.


\section{Problem Formulation}\label{sec:problem}

\subsection{Notation and Preliminaries}\label{sec:problem:prelim}

From now on, unless specified otherwise, by a dynamic graph we refer to CTDG.
A dynamic graph denoted as $\calG=(\calV,\calE)$ is a sequence of temporal interactions, where $\calV$ is the set of nodes, $\calE=\{e_{1},e_{2},\cdots,e_{|\calE|}\}$ is the set of interactions.
Each interaction $\edgei$ is a dyadic event between two nodes $u_{i},v_{i}\in\calV$ occurring at a specific time $t_{i}\in\mathbb{R}^{+}$.
For simplicity, we first assume these interactions are undirected, unattributed, and ordered chronologically (i.e., $t_{i}\leq t_{i+1}$).
Later, we will discuss how our proposed method can be extended to attributed dynamic graphs.

TGNNs are commonly trained using temporal interactions as supervision~\citep{thomas2022graph}. 
Therefore, the capability of a model is evaluated by how accurately it may predict future interactions based on historical events.
Specifically, given two nodes $u$ and $v$ at time $t$, TGNNs aim to learn their time-aware embeddings $\vech_{u}(t)$ and $\vech_{v}(t)$, where the presence of an interaction between them can be predicted with a downstream classifier, i.e., $\hat{y}_{uvt}=\textnormal{clf}(\vech_{u}(t),\vech_{v}(t))$. 
The loss function is typically formulated as:
\begin{equation}\label{eq:link_loss}
\setlength{\belowdisplayskip}{3pt} \setlength{\belowdisplayshortskip}{3pt}
\setlength{\abovedisplayskip}{3pt} \setlength{\abovedisplayshortskip}{3pt}
    \calL_{\calE}=-\sum_{(u,v,t)\in\calE}{\log{\hat{y}_{uvt}}+\mathbb{E}_{n\sim P_{n}(\calV)}\log{(1-\hat{y}_{unt})}},
\end{equation}
where $P_{n}(\calV)$ is the negative sampling distribution. 
In this paper, we refer to the classifier $\textnormal{clf}(\cdot)$ as the \textit{edge scorer} and the output $\hat{y}_{uvt}$ as the \textit{edge score} of the edge $(u,v,t)$.

\subsection{Dynamic Graph Adversarial Attack}
\subsubsection{Attacker's Goal.}


Given a dynamic graph, an attacker makes slight yet intentional changes to the data to achieve a purpose, for instance, deteriorating the prediction performance of the models.
In this work, we assume that the goal of the attacker is to degrade the dynamic link prediction performance of TGNN models by injecting temporal adversarial edges into the original edge sequence.
Formally, a corrupted dynamic graph is denoted as $\tilde{\calG}=(\calV,\calE\cup\tilde{\calE})$, where $\tilde{\calE}=\{\tilde{e}_{1},\tilde{e}_{2},\cdots,\tilde{e}_{\Delta}\}$ is the set of adversarial edges and $\Delta$ is the perturbation budget.
Our proposed attack takes place before the victim model is trained, i.e., it is a \mbox{poisoning attack.}

\subsubsection{Constraint for Unnoticeability.}

To ensure the effectiveness and stealthiness of adversarial attacks, we impose four constraints on dynamic graph attacks:



\begin{enumerate}[leftmargin=*,label=\textbf{(C\arabic*)},wide=0pt,noitemsep,topsep=0pt,parsep=0pt,partopsep=0pt]
    \item \textbf{Perturbation budget}: To control the extent of adversarial perturbations, we impose a budget constraint $\Delta$ on the attacks, limiting the number of adversarial edges to be no more than a budget, i.e., $|\tilde{\calE}|\leq\Delta$. 
    We control the budget with the perturbation rate $p$, i.e., $\Delta=\lfloor p\cdot|\calE|\rfloor$. 
    
    \item \textbf{Distribution of time}: To enhance the stealthiness, we impose a distribution constraint on the timing of the adversarial edges. 
    To this end, when adding an adversarial edge $\tilde{e}=(\tilde{u},\tilde{v},\tilde{t})$ at time $\tilde{t}$, we ensure that the time $\tilde{t}$ follows the time distribution observed in the original edges, i.e., $\tilde{t}\sim P_{t}(\calE)$, where $P_{t}$ is the time distribution\footnote{
    We use Kernel Density Estimation (KDE) to approximate the time distribution of the original edges. We employ the Gaussian kernel with a bandwidth of 0.1. The timestamps of the adversarial edges are drawn from the estimated temporal distribution.}.
    This constraint aims to align the occurrences of adversarial edges with the natural flow of interactions in dynamic graphs, making the adversarial edges less distinguishable from genuine edges.
    
    \item \textbf{Endpoints of adversarial edges}: If an adversarial edge is composed of nodes that were activated a significant time ago, it becomes easily detectable because it deviates from the typical pattern of interactions. 
    Thus, we enforce a restriction that adversarial edges can only be formed using nodes from interactions that occurred within a specific time window $W$. 
    Formally, when an adversarial edge $\tilde{e}$ is added at time $\tilde{t}$, where $t_{i}\leq\tilde{t}\leq t_{i+1}$, we select two nodes $\tilde{u}$ and $\tilde{v}$ from the set of nodes comprising the preceding $W$ edges, i.e., $\tilde{u},\tilde{v}\in\calV_{i,W}=\{u_{i-W+1},v_{i-W+1},\cdots,u_{i},v_{i}\}$.
    This constraint ensures that the adversarial edges are derived from recent interactions, making them less conspicuous.

    \item \textbf{Number of perturbations per node}: A sudden increase in the degree of nodes can make the attack easily detectable~\citep{bhatia2020midas}. 
    By limiting the number of perturbations per node to $N$, the attack becomes less noticeable.
    It is important to note that $N$ is not a fixed value for the entire graph; it can vary over time, considering the dynamics of the original interactions in dynamic graphs.
    
\end{enumerate}

\subsubsection{Attacker's Knowledge.}
We focus on limited-knowledge attacks, in which the attacker lacks knowledge about the victim model, including its architecture and trained weights.
However, the attacker does possess the same level of knowledge about the data as the victim model, allowing them to observe all the dynamic graph data $\calG$ and to study the temporal interactions within the dynamic graph.
Note that while our primary focus is on targeting TGN~\citep{rossi2020temporal}, one of the representative TGNNs, the scope is not limited to this model.
We investigate the transferability of the proposed attack, demonstrating its effectiveness for other victim models.
These settings are realistic because the attacker cannot know the details of the victim model in real situations.



\section{Proposed Attack Method: \attack}\label{sec:attack}

In this section, we introduce our proposed adversarial attack method \attack. 
Our attack exploits the edge scores ($\hat{y}_{uvt}$) to create adversarial edges, leveraging the fact that a low edge score indicates the edge is unlikely to be formed, and thus likely to corrupt the natural evolution of the dynamic graph.
The method aims to strategically generate and inject adversarial edges into the original edge sequence while adhering to the four constraints (C1-C4).

\subsection{Surrogate Model}\label{sec:attack:surrogate}

We face the challenge of limited knowledge about the victim model's parameters and architecture. 
To address this issue, we employ a surrogate model, specifically the TGN model, which serves as a substitute for the victim model. 
This enables us to develop and optimize adversarial edges without directly accessing the victim model's internals.
\subsubsection{Node Memory.}
TGN uses node memory to summarize past interactions.
When there is an interaction involving a node, the memory of that node is updated.
For each interaction $(u,v,t)$, a message from $v$ to $u$ is computed as follows:
\begin{equation}
    \vecm_{u}(t)=[\vecs_{u}(t^{-})||\vecs_{v}(t^{-})||\Phi(t-t_{u}^{-})||\vece_{uv}].
\end{equation}
Here, $\vecs_{u}(t^{-})$ is the memory of node $u$ just before time $t$, $||$ is the concatenation operator, and $\Phi(\Delta t)=\cos(\boldsymbol{\omega}\Delta t+\boldsymbol{\phi})$ is the time encoder which encodes the time interval $\Delta t=t-t_{u}^{-}$ into a vector, where $\boldsymbol{\omega}$ and $\boldsymbol{\phi}$ are learnable vectors.
Edge features are denoted as $\vece_{uv}$, but they can be ignored if not provided.
The node memory is then updated as:

\begin{equation}
    \vecs_{u}(t)=\textnormal{mem}(\vecs_{u}(t^{-}), \vecm_{u}(t)),
\end{equation}
where mem($\cdot$) is the GRU~\citep{cho2014learning} memory updater.
The memory of node $v$ is also updated in the same manner.

\subsubsection{Attention Aggregator.}
TGN employs the attention mechanism of Transformer~\citep{vaswani2017attention} to aggregate information from temporal neighbors. 
The attention aggregation of node $u$ is computed by the queries, keys, and values from its temporal neighbors $v\in\calN_{u}(t)$, where $\calN_{u}(T)=\{v:(u,v,t)\in\calE,t\leq T\}$.
The temporal embedding $\vech_{u}(t)$ of node $u$ at time $t$ can be computed by:
\begin{equation}
\setlength{\belowdisplayskip}{3pt} \setlength{\belowdisplayshortskip}{3pt}
\setlength{\abovedisplayskip}{3pt} \setlength{\abovedisplayshortskip}{3pt}
    \vech_{u}(t)=\sum_{v\in\calN_{u}(t)}{\textnormal{attn}(\vecs_{u}(t),\vecs_{v}(t),\vece_{uv},\vecv_{u}(t),\vecv_{v}(t))},
\end{equation}
where attn($\cdot$) is the temporal graph attention module~\citep{xu2020inductive}, and $\vecv_{u}(t)$ is the node features. 
The node features can also be disregarded if they are not provided.
The attention aggregator enables TGN to integrate relevant information from its temporal neighbors, considering their importance in the context of the dynamic graph at time $t$.

\subsubsection{Train the Surrogate Model.}

As discussed in Section~\ref{sec:problem:prelim}, a surrogate TGN is trained in conjunction with a link classifier, which is responsible for predicting the occurrence of future interactions based on the time-aware node embeddings.
After training, the classifier (i.e., edge scorer) estimates the likelihood of edge formation (i.e., edge score).


\subsection{Na\"ive Adversarial Edge Selection}\label{sec:attack:naive}

Given a well-trained surrogate model and its associated edge scorer, to create the adversarial edges $\tilde{\calE}=\{\tilde{e}_{1},\tilde{e}_{2},\cdots,\tilde{e}_{\Delta}\}$, we first determine the timestamps, i.e., $\tilde{t_{1}},\cdots,\tilde{t_{\Delta}}$, at which the adversarial edges are to be injected.
Ensuring stealthiness, we sample $\tilde{t_{1}},\cdots,\tilde{t_{\Delta}}\sim P_{t}(\calE)$ to align them with the time distribution of the original edges (C2).

Next, for each time sampled $\tilde{t}$, we select two nodes $\tilde{u}$ and $\tilde{v}$ based on their edge score.
To satisfy the constraint (C3), we choose two nodes from the set of nodes that comprise the previous $W$ edges.
When $t_{i}\leq\tilde{t}\leq t_{i+1}$ holds, we compute edge scores for all pairs of nodes in $\calV_{i,W}$ using the node embeddings at time $\tilde{t}$.
The node pair with the lowest edge score is chosen as the endpoints of the adversarial edge.
It is formally described as follows:
\begin{equation}
\setlength{\belowdisplayskip}{3pt} \setlength{\belowdisplayshortskip}{3pt}
\setlength{\abovedisplayskip}{3pt} \setlength{\abovedisplayshortskip}{3pt}
    (\tilde{u},\tilde{v})=\argmin_{(u,v)\in\calE(\calV_{i,W})}{\textnormal{clf}(\vech_{u}(\tilde{t}),\vech_{v}(\tilde{t}))}.
\end{equation}
Here, $\calE(\calV_{i,W})=\{(u,v)|u,v\in\calV_{i,W},u\neq v\}$ denotes all pairs of nodes within $\calV_{i,W}$.
The lowest edge score indicates that the corresponding edge is the least likely to be formed at time $\tilde{t}$.
Thus, the proposed adversarial edge selection rule may effectively perturb the temporal dynamics of the graph.
We repeat the edge selection $\Delta$ times for all timestamps $\tilde{t_{1}},\cdots,\tilde{t_{\Delta}}$, thereby satisfying the constraint (C1).

\subsection{Advanced Edge Selection on a Batch}\label{sec:attack:batch}
Selecting adversarial edges one by one is effective but it may become impractical when dealing with large datasets.
To overcome this issue, we propose an alternative where adversarial edges are selected in batches, improving efficiency.
The batch size $|B|$ is set to half of the window size $\lfloor|W|/2\rfloor$ to adhere the constraint (C3), and we generate $K=\lfloor p\cdot|B|\rfloor$ adversarial edges per batch.

For each batch $b$, we first sample $K$ timestamps from the time distribution $P_{t}(\calE)$, i.e., $\tilde{t}_{b,1},\cdots,\tilde{t}_{b,K}\sim P_{t}(\calE)$, where $t_{b,1}\leq\tilde{t}_{b,1}\leq\cdots\leq\tilde{t}_{b,K}\leq t_{b,|B|}$.
Here, $t_{b,1}$ and $t_{b,|B|}$ are the first and the last timestamps of the original edges in the batch $b$, respectively.
Let $\calV_{b-1}$ be a set of nodes constituting the edges in the previous batch $b-1$.
We then compute the edge scores for all node pairs within $\calV_{b-1}$ using the node embeddings at time $\tilde{t}_{b,1}$, which is the first timestamp of the perturbations.
For an edge $(u,v,\tilde{t}_{b,1})$, its edge score is $\hat{y}_{uv\tilde{t}_{b,1}}$.
Based on these edge scores, we select $K$ edges as the adversarial edges. 

\subsubsection{Low-K Selection.}\label{sec:attack:lowk}
The simplest way to select $K$ edges is to pick the $K$ edges with the lowest edge scores.
However, we found that this method often leads to a substantial overlap of nodes.
While it effectively corrupts information related to specific nodes, it may not sufficiently degrade the overall link prediction performance.
Moreover, this method can result in a sudden increase in high-degree nodes, which violates constraint (C4), making the attack more detectable.

\begin{figure}
    \centering
    \includegraphics[width=0.97\columnwidth]{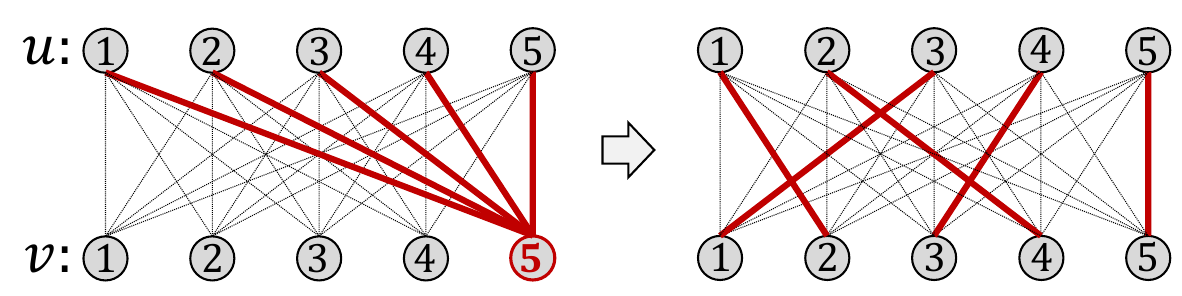}
    \caption{Comparison of Low-K selection (left) and Hungarian selection (right) when $K$ is 5. 
    The Hungarian selection provides more inconspicuous and balanced perturbations. }
    \label{fig:hungarian}
\end{figure}

\subsubsection{Hungarian Selection.}\label{sec:attack:hungarian}
To ensure the chosen edges are both impactful and inconspicuous, it is important to select adversarial edges with low edge scores and minimal overlapping endpoints.
To this end, we reframe our challenge as the problem of assignment between the nodes in $\calV_{b-1}$. 
Formally, let $x_{uv}$ be an assignment indicator where $x_{uv}=1$ if node $u$ is matched with node $v$, otherwise $x_{uv}=0$.
We only allow one-to-one matching between distinct nodes.
The optimal set of edges made within $\calV_{b-1}$, denoted as $E^{*}$, can be obtained by solving our assignment problem formulated as:

\small
\vspace{-3mm}
\begin{align}
\setlength{\belowdisplayskip}{2pt} \setlength{\belowdisplayshortskip}{2pt}
\setlength{\abovedisplayskip}{2pt} \setlength{\abovedisplayshortskip}{2pt}
    E^{*} & =\argmin_{\{(u,v)|x_{uv}=1\}}{\sum_{u\in\calV_{b-1}}\sum_{v\in\calV_{b-1}}{\hat{y}_{uv\tilde{t}_{b,1}}x_{uv}}},\\
    \textnormal{subject to \ \ } & \sum_{\mathclap{v\in\calV_{b-1}\ }}\ {x_{uv}}=\sum_{\mathclap{v\in\calV_{b-1}}}\ {x_{vu}}=1 \text{ and } {x_{uu}}=0, \forall u\in\calV_{b-1}. \nonumber
\end{align}
\normalsize
Note that, if the sets of nodes available for the source and destination nodes differ (i.e., a bipartite dynamic graph), the problem changes to a rectangular assignment problem~\citep{crouse2016implementing}.
To solve this problem, we employ the Hungarian algorithm~\citep{jonker1988shortest}, one of the most well-known solvers of an assignment problem.
Solving this problem ensures that the selected edges have non-overlapping nodes, which helps improve the unnoticeability when considering (C4).
Figure~\ref{fig:hungarian} illustrates the concept of the Hungarian selection and its impact on reducing the occurrence of overlapping nodes.
Furthermore, we apply the Hungarian algorithm $N$ times, excluding previously selected edges, to create the candidate pool consisting of the selected edges.
From this pool, we choose the $K$ edges having the lowest edge score as the adversarial edges.
Then, we randomly assign $\tilde{t}_{b,1},\cdots,\tilde{t}_{b,K}$ to the chosen edges.
This approach allows us to limit the maximum number of adversarial edges per node to $N$, ensuring balanced contamination of all nodes in the graph while maintaining stealthiness (C4).
In this paper, we maintain $N$ as the minimum number of times that all perturbations can be selected, i.e., $N=\lceil\frac{K}{E_{\max}(\calV_{b-1})}\rceil$, where $E_{\max}(\calV_{b-1})$ is the maximum number of non-overlapping edges can be made within $\calV_{b-1}$.


\section{Proposed Defense Method: \defense}\label{sec:defense}

The goal of adversarial defense is to enhance the model's robustness against perturbations and alleviate performance degradation. 
The defense for dynamic graphs is particularly challenging because the criteria for judging legitimate and adversarial edges change over time in CTDGs.
To attain this goal, we identify and eliminate potential adversarial edges from the corrupted graphs considering the evolution of dynamic graphs without prior knowledge of adversarial attacks.
Moreover, we enforce temporal smoothness on node embeddings to prevent sudden changes. 
In the following subsections, we elaborate on the proposed defense method called \defense.


\subsection{Filtering Potential Adversarial Edges}\label{sec:defense:filtering}

\defense selectively excludes potential adversarial edges from the training process. 
For this, we leverage the edge scores generated by the link classifier (i.e., edge scorer), which is trained alongside the TGNN model. 

For each edge $e=(u,v,t)$, we first compute its edge score $\hat{y}_{uvt}$, and if the edge score is less than a threshold $\tau\in(0,1)$, we classify it as a potential adversarial edge, i.e., $e\in\hat{\calE}$, where $\hat{\calE}$ is the set of potential adversarial edges.
Subsequently, we can train the model on the cleaned dynamic graph $\calG_{c}=(\calV_{c},\calE_{c})$, where $\calV_{c}$ is the set of remaining nodes and $\calE_{c}={\calE\cup\tilde{\calE}}-\hat{\calE}$, obtained by removing the set of potential adversarial edges $\hat{\calE}$ from the corrupted dynamic graph. 

However, using a fixed threshold for edge filtering could lead to two significant issues. 
First, in the initial stages of training, the model and edge scorer might not be sufficiently trained, resulting in low confidence in the edge scores. Setting a high fixed threshold during this phase could remove numerous legitimate edges, impacting the model's performance.
Second, as training progresses and the model and edge scorer become well-trained, the edge scores become more reliable. In this later stage, using a low fixed threshold could increase the likelihood of failing to filter out adversarial edges effectively, leaving the model vulnerable to adversarial attacks.
To address these concerns, we employ the cosine annealing scheduler to gradually increase the threshold from $\tau_{s}$ to $\tau_{e}$ where $0 \leq \tau_s \leq \tau_e \leq 1$.
This approach provides a more adaptive and balanced filtering approach during the training process. The cleaned dynamic graph is then trained with the loss function $\calL_{\calE_c}$ (Eq.~\eqref{eq:link_loss}).

\subsection{Temporal Smoothness}\label{sec:defense:temp}
Adversarial edges may change the node embeddings dramatically, preventing the model from learning the dynamics graphs.
We ensure the temporal smoothness of node embeddings to mitigate the impact of adversarial edges that evade edge filtering.
Maintaining temporal smoothness limits the difference between two consecutive embeddings to be relatively small, preventing the embeddings from deviating too dramatically from recent history.
To impose temporal smoothness, we introduce a regularization term as follows:
\begin{equation}
    \calL_{tmp}=-\!\!\!\!\!\! \sum_{(u,v,t)\in\calE_{c}}{\sum_{i\in\{u,v\}}{\!\!\!w(t-t_{i}^{-})\textnormal{sim}(\vech_{i}(t),\vech_{i}(t_{i}^{-}))}},
\end{equation}
where $w(\Delta t)=\exp{(-\theta\Delta t)}$ is the weight function that assigns higher weights to interactions within shorter periods and $\textnormal{sim}(\cdot)$ denotes the cosine similarity. 
We set $\theta$ to 0.01.

Finally, by integrating edge filtering and temporal regularization, the loss function of \defense is the weighted sum of $\calL_{\calE_{c}}$ and $\calL_{tmp}$, i.e., $\calL=\calL_{\calE_{c}}+\lambda\calL_{tmp}$, where $\lambda$ is a hyperparameter to control the extent of temporal smoothness.

\section{Complexity Analysis}\label{sec:complexity}
The time complexities of \attack and \defense are presented in Lemmas~\ref{lemma:attack} and~\ref{lemma:defense}, respectively.
To focus on the time complexities of the proposed methods, we denote the time complexity of TGN for the inference of a single node's embedding as $C_{tgn}$ and the time complexity of the edge scorer as $C_{clf}$.
Here, $d$ denotes the dimension of node embeddings.
The proofs are provided in Appendix B.
\begin{lemma}[Time complexity of \attack]\label{lemma:attack}
    The time complexity of \attack is $O\big((C_{tgn}+|W|\cdot C_{clf}+p\cdot|W|^{2})\cdot|\calE|\big)$.
\end{lemma}
\begin{lemma}[Time complexity of \defense]\label{lemma:defense}
    The time complexity of \defense is $O\big((C_{tgn}+C_{clf}+d)\cdot|\calE|\big)$.
\end{lemma}

\section{Experiments}\label{sec:exp}

\subsection{Experimental Setup}\label{sec:exp:setup}
\subsubsection{Datasets.}

We conduct experiments on 4 real-world datasets: Wikipedia~\citep{kumar2019predicting}, MOOC~\citep{feng2019understanding}, UCI~\citep{panzarasa2009patterns}, and Bitcoin-OTC~\citep{kumar2016edge}.
The details of them are provided in Appendix A.2.
Across all datasets, we use the same 70\%/15\%/15\% chronological splits for the train/validation/test sets as in \citep{rossi2020temporal}.
Then, to make attacked graphs, we inject perturbations on all the split sets using the attack method.

\begin{table*}[!t]
    \setlength{\tabcolsep}{5pt}
    \setlength\aboverulesep{0.5pt}
    \centering
    \scalebox{0.9}{
    \begin{tabular}{llccccccc}
        \toprule
        Dataset & Victim & Clean & \random & \pa & \jaccard\!* & \degree & \pagerank & \textbf{\attack} \\
        \midrule
        \multirow{4}{*}{Wikipedia}  & \TGN      & 80.5 $\pm$ 0.5 & 66.0 $\pm$ 0.6 & \ul{65.3 $\pm$ 1.1} & - & 66.4 $\pm$ 0.5 & 66.9 $\pm$ 0.9 & \tb{60.0 $\pm$ 1.6} \\
                                    & \JODIE    & 63.2 $\pm$ 1.2 & 36.3 $\pm$ 3.4 & 35.0 $\pm$ 2.3 & - & \tb{33.0 $\pm$ 2.1} & 35.0 $\pm$ 2.5 & \ul{33.8 $\pm$ 2.4} \\
                                    & \TGAT     & 57.0 $\pm$ 0.6 & 41.4 $\pm$ 0.6 & \tb{39.2 $\pm$ 1.0} & - & 41.1 $\pm$ 0.6 & 40.8 $\pm$ 0.5 & \ul{40.7 $\pm$ 0.5} \\
                                    & \DySAT    & 66.6 $\pm$ 0.4 & 63.0 $\pm$ 0.5 & 62.4 $\pm$ 0.5 & - & \ul{62.1 $\pm$ 0.2} & 62.8 $\pm$ 0.4 & \tb{62.0 $\pm$ 0.4} \\
        \midrule
        \multirow{4}{*}{MOOC}   & \TGN          & 61.8 $\pm$ 2.1 & 55.4 $\pm$ 1.7 & 56.0 $\pm$ 1.6 & - & \ul{50.3 $\pm$ 1.4} & 53.2 $\pm$ 1.2 & \tb{48.2 $\pm$ 1.3} \\
                                & \JODIE        & 37.4 $\pm$ 2.0 & \ul{23.8 $\pm$ 0.6} & 29.3 $\pm$ 3.9 & - & 26.2 $\pm$ 1.8 & 26.6 $\pm$ 1.6 & \tb{22.8 $\pm$ 1.9} \\
                                & \TGAT         & 11.8 $\pm$ 0.1 & 10.2 $\pm$ 0.2 & 10.0 $\pm$ 0.1 & - & 10.7 $\pm$ 0.2 & \ul{9.6 $\pm$ 0.1} & \tb{9.3 $\pm$ 0.2} \\
                                & \DySAT        & 18.4 $\pm$ 0.1 & 14.5 $\pm$ 0.1 & 14.5 $\pm$ 0.3 & - & 14.5 $\pm$ 0.2 & \ul{14.4 $\pm$ 0.2} & \tb{14.3 $\pm$ 0.2} \\
        \midrule
        \multirow{4}{*}{UCI}    & \TGN      & 44.2 $\pm$ 0.4 & 42.5 $\pm$ 0.3 & 44.2 $\pm$ 0.9 & 43.0 $\pm$ 0.5 & 42.5 $\pm$ 0.5 & 42.9 $\pm$ 0.8 & \tb{41.6 $\pm$ 0.3} \\
                                & \JODIE    & 24.9 $\pm$ 0.4 & 21.9 $\pm$ 0.5 & 22.3 $\pm$ 0.4 & 21.9 $\pm$ 0.2 & 21.8 $\pm$ 0.6 & \tb{21.4 $\pm$ 0.5} & \ul{21.7 $\pm$ 0.2} \\
                                & \TGAT     & 16.3 $\pm$ 0.3 & 15.5 $\pm$ 0.3 & \ul{14.7 $\pm$ 0.1} & 14.9 $\pm$ 0.2 & 14.7 $\pm$ 0.2 & 15.7 $\pm$ 0.4 & \tb{14.6 $\pm$ 0.2} \\
                                & \DySAT    & 71.3 $\pm$ 0.9 & 65.3 $\pm$ 0.7 & \ul{65.3 $\pm$ 0.3} & 65.3 $\pm$ 0.5 & 65.4 $\pm$ 0.3 & 65.3 $\pm$ 0.6 & \tb{64.7 $\pm$ 0.4} \\
        \midrule
        \multirow{4}{*}{Bitcoin}    & \TGN      & 48.3 $\pm$ 0.7 & 48.0 $\pm$ 1.3 & 48.0 $\pm$ 1.0 & 48.4 $\pm$ 0.9 & 46.8 $\pm$ 0.5 & \ul{46.2 $\pm$ 0.7} & \tb{45.2 $\pm$ 1.4} \\
                                    & \JODIE    & 34.6 $\pm$ 0.7 & 31.4 $\pm$ 0.5 & 30.3 $\pm$ 0.6 & 32.1 $\pm$ 0.6 & 29.2 $\pm$ 0.8 & \tb{27.9 $\pm$ 0.9} & \ul{28.9 $\pm$ 0.7} \\
                                    & \TGAT     & 21.8 $\pm$ 0.4 & 18.6 $\pm$ 0.3 & 17.3 $\pm$ 0.3 & 18.5 $\pm$ 0.4 & 17.1 $\pm$ 0.4 & \tb{16.6 $\pm$ 0.6} & \ul{17.1 $\pm$ 0.2} \\
                                    & \DySAT    & 84.7 $\pm$ 0.7 & 79.0 $\pm$ 0.8 & \ul{78.9 $\pm$ 0.2} & 79.4 $\pm$ 0.6 & 79.0 $\pm$ 0.4 & 79.0 $\pm$ 0.2 & \tb{78.8 $\pm$ 0.2} \\
        \midrule    
        \multicolumn{2}{c}{A.P.D.$\downarrow$} & - & -14.7\% & -14.9\% & -7.5\% &-16.3\% & \ul{-16.4\%} & \tb{-18.9\%} \\         
        \bottomrule             
        \multicolumn{9}{l}{ *\jaccard cannot be applied to bipartite graphs.}
    \end{tabular}}
    \caption{MRR at a perturbation rate $p=0.3$. 
    For each victim model, the best and second-best performances are highlighted in boldface and underlined, respectively. 
    A.P.D. stands for average performance degradation.}
    \label{tab:attack}
\end{table*}

\subsubsection{Victim Models.}

We assess the proposed methods using 4 TGNNs: TGN~\citep{rossi2020temporal}, JODIE~\citep{kumar2019predicting}, TGAT~\citep{xu2020inductive}, and DySAT\footnote{Technically, DySAT is a discrete-time TGNN.}~\citep{sankar2018dynamic} (see Section~\ref{sec:related} for their summaries).
Our evaluation mainly focuses on TGN (see Section~\ref{sec:attack:surrogate} for its details), while the other models are utilized to verify the transferability of our proposed attack.

\subsubsection{Baselines.}
We compare \attack and \defense with various baseline approaches including 5 edge-perturbation-based adversarial attacks and 2 defense methods. 
A detailed description of these baselines is provided in Appendix A.3.

\subsubsection{Evaluation Protocols.}
We adopt the Mean Reciprocal Rank (MRR) and hit rate (specifically, Hit@10) of the true edge among 100 randomly selected negative edges (without replacement) as a validation metric.
Note that, all metrics are normalized by multiplying 100.

Because our settings involve perturbations in both the validation and test sets, we evaluate the metrics for all edges (including both ground truth and adversarial edges) in the validation set.
For the test set, we only evaluate the metrics for the ground truth edges to be a fair comparison.
Especially, since \defense and \tgncosine employ edge filtering, the evaluation process differs from others; for the validation set, we assess the metrics for the remaining edges after edge filtering, on the other hand, for the test set, we evaluate the metrics for the ground-truth edges before edge filtering.
All the results are averaged over 5 runs. 
The results evaluated by Hit@10 are provided in Appendix C.

\subsubsection{Implementation Details.}
Details on the implementation and hyperparameters of all victim models, attack models, and defense models are provided in Appendix A.4.

\begin{figure}
    \includegraphics[width=0.6\columnwidth]{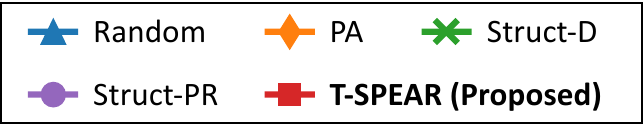}
    \centering
    \begin{subfigure}{0.495\columnwidth}
        \includegraphics[width=\columnwidth]{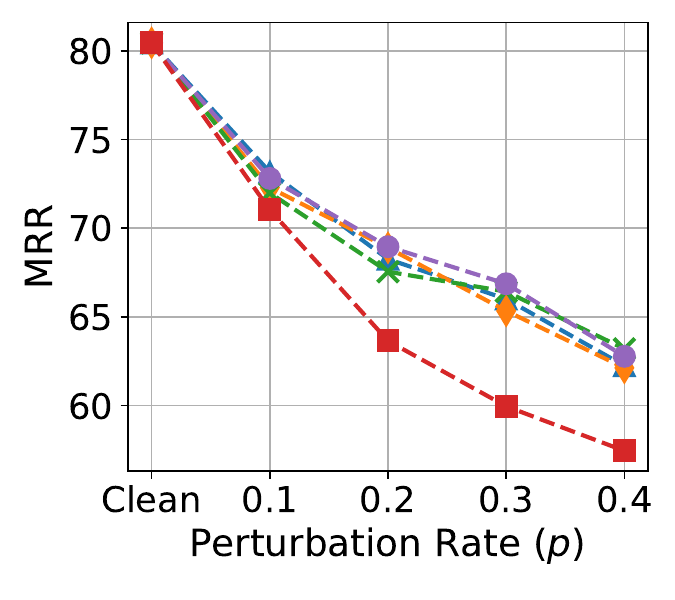}
    \end{subfigure}
    \begin{subfigure}{0.495\columnwidth}
        \includegraphics[width=\columnwidth]{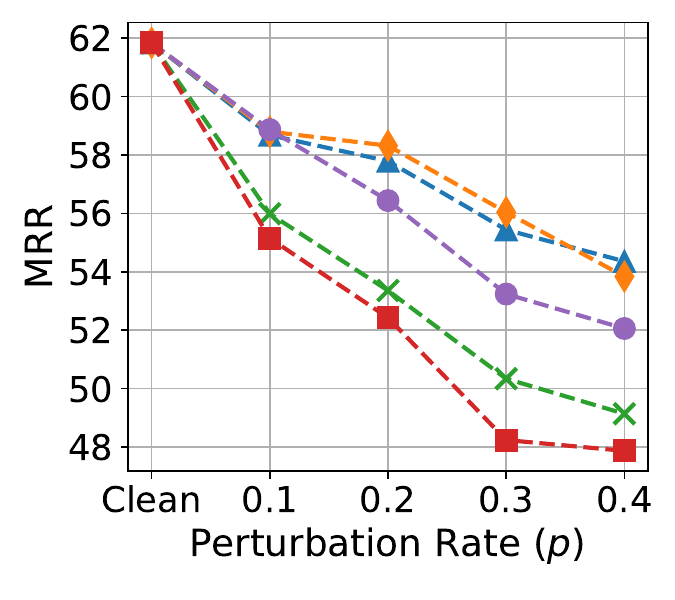}
    \end{subfigure}
    \caption{Link prediction performance of TGN on Wikipedia (left) and MOOC (right) as the perturbation rate increases.}
    \label{fig:performance}
\end{figure}

\begin{table*}[!t]
    \setlength{\tabcolsep}{3.5pt}
    \setlength\aboverulesep{0.5pt}
    \centering
    \scalebox{0.9}{
    \begin{tabular}{lccccccccccc}
        \toprule
        \multirow{2}{*}[-1mm]{Model} & \multicolumn{3}{c}{Wikipedia (Clean: 80.5 $\pm$ 0.5)} & \multicolumn{3}{c}{MOOC (Clean: 61.8 $\pm$ 2.1)} & \multicolumn{3}{c}{UCI (Clean: 44.2 $\pm$ 0.4)} & \multirow{2}{*}[-1mm]{A.P.G.$\uparrow$} \\
        \cmidrule(lr){2-4} \cmidrule(lr){5-7} \cmidrule(lr){8-10} 
        & $p=0.1$ & $p=0.2$ & $p=0.3$ & $p=0.1$ & $p=0.2$ & $p=0.3$ & $p=0.1$ & $p=0.2$ & $p=0.3$ \\
        \midrule

        \TGN            & 73.2 $\pm$ 1.1        & 68.3 $\pm$ 0.9       & 66.0 $\pm$ 0.6        & 58.7 $\pm$ 1.0        & 57.8 $\pm$ 0.9        & 55.4 $\pm$ 1.7        & 44.5 $\pm$ 0.3 & 43.5 $\pm$ 0.7 & 42.5 $\pm$ 0.3 & - \\
        \tgnsvd         & 65.3 $\pm$ 1.1 & 63.9 $\pm$ 1.2 & 60.6 $\pm$ 0.9 & 58.8 $\pm$ 0.3 & 58.0 $\pm$ 1.4 & 55.0 $\pm$ 1.1 & \tb{44.5 $\pm$ 0.7} & \tb{44.3 $\pm$ 0.5} & \tb{44.7 $\pm$ 0.7} & -2.1\% \\
        \tgncosine      & 75.8 $\pm$ 0.5 & 71.6 $\pm$ 0.7 & 68.8 $\pm$ 0.4 & 53.7 $\pm$ 2.7 & 52.8 $\pm$ 4.1 & 54.0 $\pm$ 3.2 & 43.2 $\pm$ 0.6 & 43.0 $\pm$ 0.6 & 43.8 $\pm$ 0.4 & -0.9\% \\
        \midrule
        \textbf{\defensefilter}  & \tb{77.4 $\pm$ 0.5} & \tb{77.4 $\pm$ 0.2} & \ul{76.0 $\pm$ 0.3} & \ul{59.2 $\pm$ 1.0} & \tb{58.8 $\pm$ 1.3} & \ul{58.0 $\pm$ 0.6} & \ul{44.1 $\pm$ 0.2} & 44.2 $\pm$ 0.6 & 44.2 $\pm$ 0.7 & \ul{+5.1\%} \\
        \textbf{\defense}        & \ul{77.3 $\pm$ 0.5} & \ul{77.0 $\pm$ 0.3} & \tb{76.5 $\pm$ 0.5} & \tb{59.4 $\pm$ 1.2} & \ul{58.4 $\pm$ 1.1} & \tb{58.4 $\pm$ 0.6} & 43.9 $\pm$ 0.6 & \ul{44.3 $\pm$ 0.7} & \ul{44.6 $\pm$ 0.8} & \tb{+5.3\%} \\
        
        \bottomrule             
    \end{tabular}}
    \caption{MRR under \random attack. A.P.G. stands for average performance gain over the na\"ive TGN.}
    \label{tab:defense:random}
\end{table*}

\begin{table*}[!t]
    \setlength{\tabcolsep}{3.5pt}
    \setlength\aboverulesep{0.5pt}
    \centering
    \scalebox{0.9}{
    \begin{tabular}{lccccccccccc}
        \toprule
        \multirow{2}{*}[-1mm]{Model} & \multicolumn{3}{c}{Wikipedia (Clean: 80.5 $\pm$ 0.5)} & \multicolumn{3}{c}{MOOC (Clean: 61.8 $\pm$ 2.1)} & \multicolumn{3}{c}{UCI (Clean: 44.2 $\pm$ 0.4)} & \multirow{2}{*}[-1mm]{A.P.G.$\uparrow$} \\
        \cmidrule(lr){2-4} \cmidrule(lr){5-7} \cmidrule(lr){8-10} 
        & $p=0.1$ & $p=0.2$ & $p=0.3$ & $p=0.1$ & $p=0.2$ & $p=0.3$ & $p=0.1$ & $p=0.2$ & $p=0.3$ \\
        \midrule

        \TGN            & 71.0 $\pm$ 2.0        & 63.7 $\pm$ 2.6        & 60.0 $\pm$ 1.6        & 55.2 $\pm$ 0.6        & 52.4 $\pm$ 0.4        & 48.2 $\pm$ 1.3        & 43.6 $\pm$ 1.0 & 41.7 $\pm$ 0.5 & 41.6 $\pm$ 0.3 & - \\
        \tgnsvd         & 67.6 $\pm$ 0.9 & 61.2 $\pm$ 1.4 & 58.1 $\pm$ 1.8 & 55.7 $\pm$ 1.3 & 50.6 $\pm$ 2.6 & 50.1 $\pm$ 2.1 & \tb{46.0 $\pm$ 0.5} & \ul{45.5 $\pm$ 0.4} & \tb{46.5 $\pm$ 0.4} & +1.8\% \\
        \tgncosine      & 72.3 $\pm$ 0.8 & 62.9 $\pm$ 2.0 & 59.6 $\pm$ 2.1 & 53.5 $\pm$ 3.1 & 47.6 $\pm$ 1.7 & 45.2 $\pm$ 2.1 & 43.7 $\pm$ 0.4 & 44.1 $\pm$ 0.3 & 43.2 $\pm$ 1.3 & -1.0\% \\
        \midrule
        \textbf{\defensefilter}  & \tb{77.3 $\pm$ 0.5} & \ul{76.8 $\pm$ 0.3} & \ul{75.3 $\pm$ 0.4} & \ul{58.5 $\pm$ 0.5} & \ul{55.7 $\pm$ 0.9} & \tb{54.5 $\pm$ 0.9} & 44.4 $\pm$ 0.8 & \tb{45.5 $\pm$ 0.3} & 44.7 $\pm$ 0.4 & \ul{+11.0\%} \\
        \textbf{\defense}        & \ul{77.1 $\pm$ 0.2} & \tb{76.9 $\pm$ 0.3} & \tb{76.1 $\pm$ 0.3} & \tb{58.8 $\pm$ 0.7} & \tb{56.3 $\pm$ 1.2} & \ul{53.9 $\pm$ 1.1} & \ul{44.9 $\pm$ 0.3} & 44.3 $\pm$ 0.5 & \ul{45.4 $\pm$ 0.8} & \tb{+11.2\%} \\
        
        \bottomrule             
    \end{tabular}}
    \caption{MRR under \attack attack. A.P.G. stands for average performance gain over the na\"ive TGN.}
    \label{tab:defense:proposed}
\end{table*}

\subsection{Attack Performance}

We conduct a comprehensive comparison of our attack method against various baseline approaches. 
Table~\ref{tab:attack} summarizes the link prediction performance of all victim models under different attack methods with a perturbation rate of 0.3. 
Across all datasets and victim models, our proposed attack consistently outperforms the baselines (lower values indicate better attack performance). 
Overall, \attack shows the most substantial performance degradation, averaging -18.9\%. 
Moreover, our attack significantly degrades the performance of other TGNNs (i.e., JODIE, TGAT, and DySAT), which differ from the surrogate model (i.e., TGN) used by \attack.
In Figure~\ref{fig:performance}, we observe the drop in link prediction performance of TGN on Wikipedia and MOOC datasets as the number of adversarial edges increases.
\attack consistently achieves a more significant reduction in performance compared to other baselines in any perturbation rates.
In Appendix C, we assess the efficacy of our attack under the scenario where the attacker lacks knowledge about the embedding size of the victim model.

\attack demonstrates particular effectiveness when the victim models are TGN and JODIE.
This is because these models incorporate a memory module to retain historical node interactions.
As the perturbations generated by our attack are the least likely edges to be formed when considering the graphs' dynamics, these perturbations effectively distort the node memory compared to baselines, which only consider basic graph statistics such as node degree, PageRank centrality, and the number of common neighbors.
On the other hand, TGAT and DySAT operate without relying on node memory. 
They utilize a self-attention mechanism for adaptive message aggregation, which makes them less susceptible to perturbations over long periods. 
Nevertheless, since our attack injects perturbations considering the \textit{current dynamics} of graphs, it can substantially reduce performance.

Additionally, we found that connecting two nodes with low node centrality (in terms of degree and PageRank centrality) is also an effective attack. 
This demonstrates that low-degree attacks~\citep{hussain2021structack,ma2020towards}, which are effective in static graphs, also exhibit partial effectiveness in dynamic graphs.




\subsection{Defense Performance}

To demonstrate the generality of our proposed defense method, we evaluate its robustness under two distinct attacks: \random and \attack.
We first poison dynamic graphs using each attack, and then train the \defense and its baselines on the poisoned graphs to evaluate their link prediction performance.
Tables~\ref{tab:defense:random} and~\ref{tab:defense:proposed} summarize the performance of the defense models against the \random and \attack, respectively.
\defensefilter, which uses only edge filtering, is a variant of \defense to validate the effect of edge filtering itself. 
Notably, \defense and \defensefilter consistently improve the link prediction performance across different attack methods, datasets, and perturbation rates. 
It verifies that our proposed defense method is a suitable defense method for real-world situations where the attacker's details are unknown.
The primary performance improvement comes from edge filtering, but enforcing temporal smoothness on node embeddings provides additional enhancement.
On the other hand, baselines exhibit insufficient defense performance (even if carefully optimized), except for the UCI dataset.
It indicates that cosine similarity is not appropriate to filter out adversarial edges in bipartite graphs, and applying SVD to the entire training set can not capture the reliability of edges changing over time. 
To sum up, it is not straightforward to extend static defense methods to CTDGs.

Table~\ref{tab:defense:f1} shows the AUROC of filtering-based defense models in classifying adversarial edges at $p=0.3$, demonstrating that our proposed defense methods more accurately discriminate adversarial edges compared to the baselines.

\begin{table}[!t]
    \setlength{\tabcolsep}{4pt}
    \setlength\aboverulesep{0.5pt}
    \centering
    \scalebox{0.9}{
    \begin{tabular}{llccc}
        \toprule
        Attack & Model & Wikipedia & MOOC & UCI \\
        \midrule
        \multirow{3}{*}[0mm]{\random}   & \tgncosine        & 0.605 & 0.633 & 0.495 \\
                                        & \textbf{\defensefilter}    & \ul{0.848} & \ul{0.697} & \ul{0.531} \\
                                        & \textbf{\defense}          & \tb{0.850} & \tb{0.709} & \tb{0.534} \\
        \midrule
        \multirow{3}{*}[0mm]{\attack}   & \tgncosine        & 0.474 & 0.556 & 0.518 \\
                                        & \textbf{\defensefilter}    & \ul{0.858} & \tb{0.679} & \ul{0.608} \\
                                        & \textbf{\defense}          & \tb{0.864} & \ul{0.674} & \tb{0.620} \\
        \bottomrule
    \end{tabular}
    }
    \normalsize
    \caption{AUROC of adversarial edge classification at perturbation rate $p=0.3$.}
    \label{tab:defense:f1}
\end{table}

\section{Conclusion}

In this paper, we propose \attack, an effective poisoning adversarial attack method for TGNN-based link prediction on CTDGs.
While adhering to the four constraints for unnoticeability, we strategically generate adversarial edges, which are unlikely to be formed, and thus likely to corrupt the natural evolution of the dynamic graphs. Additionally, to mitigate the impact of adversarial attacks, we propose \defense, a robust training method for TGNNs.
By employing the edge filtering technique and enforcing temporal robustness to node embeddings, we improve the robustness of the victim mode.
Experimental results demonstrate that \attack significantly degrades victim model performance and is transferable across various TGNNs. 
Moreover, \defense yields significant performance improvements over the basic TGN model under different adversarial attacks.

\section*{Acknowledgements}
This work was supported by Institute of Information \& Communications Technology Planning \& Evaluation (IITP)
grant funded by the Korea government (MSIT) (No. 2022-0-00157,
Robust, Fair, Extensible Data-Centric Continual Learning) (No. 2019-
0-00075, Artificial Intelligence Graduate School Program (KAIST)).

\bibliography{aaai24}

\begin{thebibliography}{47}
\providecommand{\natexlab}[1]{#1}

\bibitem[{Bhatia et~al.(2020)Bhatia, Hooi, Yoon, Shin, and
  Faloutsos}]{bhatia2020midas}
Bhatia, S.; Hooi, B.; Yoon, M.; Shin, K.; and Faloutsos, C. 2020.
\newblock Midas: Microcluster-based detector of anomalies in edge streams.
\newblock In \emph{Proceedings of the AAAI Conference on Artificial
  Intelligence}, volume~34, 3242--3249.

\bibitem[{Bojchevski and G{\"u}nnemann(2019)}]{bojchevski2019adversarial}
Bojchevski, A.; and G{\"u}nnemann, S. 2019.
\newblock Adversarial attacks on node embeddings via graph poisoning.
\newblock In \emph{International Conference on Machine Learning}, 695--704.
  PMLR.

\bibitem[{Chen et~al.(2021)Chen, Zhang, Chen, Du, and Xuan}]{chen2021time}
Chen, J.; Zhang, J.; Chen, Z.; Du, M.; and Xuan, Q. 2021.
\newblock Time-aware gradient attack on dynamic network link prediction.
\newblock \emph{IEEE Transactions on Knowledge and Data Engineering}.

\bibitem[{Cho et~al.(2014)Cho, Merrienboer, Gulcehre, Bougares, Schwenk, and
  Bengio}]{cho2014learning}
Cho, K.; Merrienboer, B.; Gulcehre, C.; Bougares, F.; Schwenk, H.; and Bengio,
  Y. 2014.
\newblock Learning Phrase Representations using RNN Encoder-Decoder for
  Statistical Machine Translation.
\newblock In \emph{Proceedings of the 2014 Conference on Empirical Methods in
  Natural Language Processing}.

\bibitem[{Cong et~al.(2023)Cong, Zhang, Kang, Yuan, Wu, Zhou, Tong, and
  Mahdavi}]{cong2023we}
Cong, W.; Zhang, S.; Kang, J.; Yuan, B.; Wu, H.; Zhou, X.; Tong, H.; and
  Mahdavi, M. 2023.
\newblock Do We Really Need Complicated Model Architectures For Temporal
  Networks?
\newblock \emph{arXiv:2302.11636}.

\bibitem[{Crouse(2016)}]{crouse2016implementing}
Crouse, D.~F. 2016.
\newblock On implementing 2D rectangular assignment algorithms.
\newblock \emph{IEEE Transactions on Aerospace and Electronic Systems}, 52(4):
  1679--1696.

\bibitem[{Dai et~al.(2023)Dai, Lin, Zhang, and Wang}]{dai2023unnoticeable}
Dai, E.; Lin, M.; Zhang, X.; and Wang, S. 2023.
\newblock Unnoticeable backdoor attacks on graph neural networks.
\newblock In \emph{Proceedings of the ACM Web Conference 2023}, 2263--2273.

\bibitem[{Dai et~al.(2018)Dai, Li, Tian, Huang, Wang, Zhu, and
  Song}]{dai2018adversarial}
Dai, H.; Li, H.; Tian, T.; Huang, X.; Wang, L.; Zhu, J.; and Song, L. 2018.
\newblock Adversarial attack on graph structured data.
\newblock In \emph{International Conference on Machine Learning}, 1115--1124.
  PMLR.

\bibitem[{Entezari et~al.(2020)Entezari, Al-Sayouri, Darvishzadeh, and
  Papalexakis}]{entezari2020all}
Entezari, N.; Al-Sayouri, S.~A.; Darvishzadeh, A.; and Papalexakis, E.~E. 2020.
\newblock All you need is low (rank) defending against adversarial attacks on
  graphs.
\newblock In \emph{Proceedings of the 13th International Conference on Web
  Search and Data Mining}, 169--177.

\bibitem[{Feng, Tang, and Liu(2019)}]{feng2019understanding}
Feng, W.; Tang, J.; and Liu, T.~X. 2019.
\newblock Understanding dropouts in MOOCs.
\newblock In \emph{Proceedings of the AAAI Conference on Artificial
  Intelligence}, volume~33, 517--524.

\bibitem[{Fursov et~al.(2021)Fursov, Morozov, Kaploukhaya, Kovtun,
  Rivera-Castro, Gusev, Babaev, Kireev, Zaytsev, and
  Burnaev}]{fursov2021adversarial}
Fursov, I.; Morozov, M.; Kaploukhaya, N.; Kovtun, E.; Rivera-Castro, R.; Gusev,
  G.; Babaev, D.; Kireev, I.; Zaytsev, A.; and Burnaev, E. 2021.
\newblock Adversarial attacks on deep models for financial transaction records.
\newblock In \emph{Proceedings of the 27th ACM SIGKDD Conference on Knowledge
  Discovery \& Data Mining}, 2868--2878.

\bibitem[{Hooi et~al.(2016)Hooi, Song, Beutel, Shah, Shin, and
  Faloutsos}]{hooi2016fraudar}
Hooi, B.; Song, H.~A.; Beutel, A.; Shah, N.; Shin, K.; and Faloutsos, C. 2016.
\newblock Fraudar: Bounding graph fraud in the face of camouflage.
\newblock In \emph{Proceedings of the 22nd ACM SIGKDD International Conference
  on Knowledge Discovery \& Data Mining}, 895--904.

\bibitem[{Hou et~al.(2019)Hou, Fan, Zhang, Ye, Lei, Wan, Wang, Xiong, and
  Shao}]{hou2019alphacyber}
Hou, S.; Fan, Y.; Zhang, Y.; Ye, Y.; Lei, J.; Wan, W.; Wang, J.; Xiong, Q.; and
  Shao, F. 2019.
\newblock $\alpha$cyber: Enhancing robustness of android malware detection
  system against adversarial attacks on heterogeneous graph based model.
\newblock In \emph{Proceedings of the 28th ACM International Conference on
  Information and Knowledge Management}, 609--618.

\bibitem[{Hussain et~al.(2021)Hussain, Duricic, Lex, Helic, Strohmaier, and
  Kern}]{hussain2021structack}
Hussain, H.; Duricic, T.; Lex, E.; Helic, D.; Strohmaier, M.; and Kern, R.
  2021.
\newblock Structack: Structure-based adversarial attacks on graph neural
  networks.
\newblock In \emph{Proceedings of the 32nd ACM Conference on Hypertext and
  Social Media}, 111--120.

\bibitem[{Jin, Li, and Pan(2022)}]{jin2022neural}
Jin, M.; Li, Y.-F.; and Pan, S. 2022.
\newblock Neural temporal walks: Motif-aware representation learning on
  continuous-time dynamic graphs.
\newblock \emph{Advances in Neural Information Processing Systems}, 35:
  19874--19886.

\bibitem[{Jin et~al.(2021)Jin, Li, Xu, Wang, Ji, Aggarwal, and
  Tang}]{jin2021adversarial}
Jin, W.; Li, Y.; Xu, H.; Wang, Y.; Ji, S.; Aggarwal, C.; and Tang, J. 2021.
\newblock Adversarial attacks and defenses on graphs.
\newblock \emph{ACM SIGKDD Explorations Newsletter}, 22(2): 19--34.

\bibitem[{Jin et~al.(2020)Jin, Ma, Liu, Tang, Wang, and Tang}]{jin2020graph}
Jin, W.; Ma, Y.; Liu, X.; Tang, X.; Wang, S.; and Tang, J. 2020.
\newblock Graph structure learning for robust graph neural networks.
\newblock In \emph{Proceedings of the 26th ACM SIGKDD International Conference
  on Knowledge Discovery \& Data Mining}, 66--74.

\bibitem[{Jonker and Volgenant(1988)}]{jonker1988shortest}
Jonker, R.; and Volgenant, T. 1988.
\newblock A shortest augmenting path algorithm for dense and sparse linear
  assignment problems.
\newblock In \emph{DGOR/NSOR}, 622--622. Springer.

\bibitem[{Kazemi et~al.(2020)Kazemi, Goel, Jain, Kobyzev, Sethi, Forsyth, and
  Poupart}]{kazemi2020representation}
Kazemi, S.~M.; Goel, R.; Jain, K.; Kobyzev, I.; Sethi, A.; Forsyth, P.; and
  Poupart, P. 2020.
\newblock Representation learning for dynamic graphs: A survey.
\newblock \emph{The Journal of Machine Learning Research}, 21(1): 2648--2720.

\bibitem[{Kingma and Ba(2014)}]{kingma2014adam}
Kingma, D.~P.; and Ba, J. 2014.
\newblock Adam: A method for stochastic optimization.
\newblock \emph{arXiv:1412.6980}.

\bibitem[{Kumar et~al.(2016)Kumar, Spezzano, Subrahmanian, and
  Faloutsos}]{kumar2016edge}
Kumar, S.; Spezzano, F.; Subrahmanian, V.; and Faloutsos, C. 2016.
\newblock Edge weight prediction in weighted signed networks.
\newblock In \emph{2016 IEEE 16th International Conference on Data Mining},
  221--230. IEEE.

\bibitem[{Kumar, Zhang, and Leskovec(2019)}]{kumar2019predicting}
Kumar, S.; Zhang, X.; and Leskovec, J. 2019.
\newblock Predicting dynamic embedding trajectory in temporal interaction
  networks.
\newblock In \emph{Proceedings of the 25th ACM SIGKDD International Conference
  on Knowledge Discovery \& Data Mining}, 1269--1278.

\bibitem[{Li et~al.(2022)Li, Li, Feng, Yuan, Wang, and Zha}]{li2022robust}
Li, H.; Li, C.; Feng, K.; Yuan, Y.; Wang, G.; and Zha, H. 2022.
\newblock Robust knowledge adaptation for dynamic graph neural networks.
\newblock \emph{arXiv:2207.10839}.

\bibitem[{Liben-Nowell and Kleinberg(2003)}]{liben2003link}
Liben-Nowell, D.; and Kleinberg, J. 2003.
\newblock The link prediction problem for social networks.
\newblock In \emph{Proceedings of the 12th International Conference on
  Information and Knowledge Management}, 556--559.

\bibitem[{Liu et~al.(2019)Liu, Si, Zhu, Li, and Hsieh}]{liu2019unified}
Liu, X.; Si, S.; Zhu, X.; Li, Y.; and Hsieh, C.-J. 2019.
\newblock A unified framework for data poisoning attack to graph-based
  semi-supervised learning.
\newblock \emph{arXiv:1910.14147}.

\bibitem[{Ma, Ding, and Mei(2020)}]{ma2020towards}
Ma, J.; Ding, S.; and Mei, Q. 2020.
\newblock Towards more practical adversarial attacks on graph neural networks.
\newblock \emph{Advances in Neural Information Processing Systems}, 33:
  4756--4766.

\bibitem[{Mujkanovic et~al.(2022)Mujkanovic, Geisler, G{\"u}nnemann, and
  Bojchevski}]{mujkanovic2022defenses}
Mujkanovic, F.; Geisler, S.; G{\"u}nnemann, S.; and Bojchevski, A. 2022.
\newblock Are Defenses for Graph Neural Networks Robust?
\newblock \emph{Advances in Neural Information Processing Systems}, 35:
  8954--8968.

\bibitem[{Panzarasa, Opsahl, and Carley(2009)}]{panzarasa2009patterns}
Panzarasa, P.; Opsahl, T.; and Carley, K.~M. 2009.
\newblock Patterns and dynamics of users' behavior and interaction: Network
  analysis of an online community.
\newblock \emph{Journal of the American Society for Information Science and
  Technology}, 60(5): 911--932.

\bibitem[{Pareja et~al.(2020)Pareja, Domeniconi, Chen, Ma, Suzumura, Kanezashi,
  Kaler, Schardl, and Leiserson}]{pareja2020evolvegcn}
Pareja, A.; Domeniconi, G.; Chen, J.; Ma, T.; Suzumura, T.; Kanezashi, H.;
  Kaler, T.; Schardl, T.; and Leiserson, C. 2020.
\newblock Evolvegcn: Evolving graph convolutional networks for dynamic graphs.
\newblock In \emph{Proceedings of the AAAI Conference on Artificial
  Intelligence}, volume~34, 5363--5370.

\bibitem[{Paszke et~al.(2019)Paszke, Gross, Massa, Lerer, Bradbury, Chanan,
  Killeen, Lin, Gimelshein, Antiga et~al.}]{paszke2019pytorch}
Paszke, A.; Gross, S.; Massa, F.; Lerer, A.; Bradbury, J.; Chanan, G.; Killeen,
  T.; Lin, Z.; Gimelshein, N.; Antiga, L.; et~al. 2019.
\newblock Pytorch: An imperative style, high-performance deep learning library.
\newblock \emph{Advances in Neural Information Processing Systems}, 32.

\bibitem[{Rossi et~al.(2020)Rossi, Chamberlain, Frasca, Eynard, Monti, and
  Bronstein}]{rossi2020temporal}
Rossi, E.; Chamberlain, B.; Frasca, F.; Eynard, D.; Monti, F.; and Bronstein,
  M. 2020.
\newblock Temporal graph networks for deep learning on dynamic graphs.
\newblock \emph{arXiv:2006.10637}.

\bibitem[{Sankar et~al.(2018)Sankar, Wu, Gou, Zhang, and
  Yang}]{sankar2018dynamic}
Sankar, A.; Wu, Y.; Gou, L.; Zhang, W.; and Yang, H. 2018.
\newblock Dynamic graph representation learning via self-attention networks.
\newblock \emph{arXiv:1812.09430}.

\bibitem[{Seo et~al.(2018)Seo, Defferrard, Vandergheynst, and
  Bresson}]{seo2018structured}
Seo, Y.; Defferrard, M.; Vandergheynst, P.; and Bresson, X. 2018.
\newblock Structured sequence modeling with graph convolutional recurrent
  networks.
\newblock In \emph{International Conference on Neural Information Processing},
  362--373. Springer.

\bibitem[{Sharma et~al.(2023)Sharma, Trivedi, Sridhar, and
  Kumar}]{sharma2023temporal}
Sharma, K.; Trivedi, R.; Sridhar, R.; and Kumar, S. 2023.
\newblock Temporal Dynamics-Aware Adversarial Attacks on Discrete-Time Dynamic
  Graph Models.
\newblock In \emph{Proceedings of the 29th ACM SIGKDD Conference on Knowledge
  Discovery \& Data Mining}, 2023--2035.

\bibitem[{Sun et~al.(2022)Sun, Dou, Yang, Zhang, Wang, Philip, He, and
  Li}]{sun2022adversarial}
Sun, L.; Dou, Y.; Yang, C.; Zhang, K.; Wang, J.; Philip, S.~Y.; He, L.; and Li,
  B. 2022.
\newblock Adversarial attack and defense on graph data: A survey.
\newblock \emph{IEEE Transactions on Knowledge and Data Engineering}.

\bibitem[{Thomas et~al.(2022)Thomas, Moallemy-Oureh, Beddar-Wiesing, and
  Holzh{\"u}ter}]{thomas2022graph}
Thomas, J.~M.; Moallemy-Oureh, A.; Beddar-Wiesing, S.; and Holzh{\"u}ter, C.
  2022.
\newblock Graph neural networks designed for different graph types: A survey.
\newblock \emph{arXiv:2204.03080}.

\bibitem[{Vaswani et~al.(2017)Vaswani, Shazeer, Parmar, Uszkoreit, Jones,
  Gomez, Kaiser, and Polosukhin}]{vaswani2017attention}
Vaswani, A.; Shazeer, N.; Parmar, N.; Uszkoreit, J.; Jones, L.; Gomez, A.~N.;
  Kaiser, {\L}.; and Polosukhin, I. 2017.
\newblock Attention is all you need.
\newblock \emph{Advances in Neural Information Processing Systems}, 30.

\bibitem[{Wang et~al.(2019)Wang, Zheng, Ye, Gan, Li, Song, Zhou, Ma, Yu, Gai
  et~al.}]{wang2019deep}
Wang, M.; Zheng, D.; Ye, Z.; Gan, Q.; Li, M.; Song, X.; Zhou, J.; Ma, C.; Yu,
  L.; Gai, Y.; et~al. 2019.
\newblock Deep graph library: A graph-centric, highly-performant package for
  graph neural networks.
\newblock \emph{arXiv:1909.01315}.

\bibitem[{Wang et~al.(2021{\natexlab{a}})Wang, Lyu, Li, Xia, Yang, Wang, Wang,
  Cui, Yang, Sun et~al.}]{wang2021apan}
Wang, X.; Lyu, D.; Li, M.; Xia, Y.; Yang, Q.; Wang, X.; Wang, X.; Cui, P.;
  Yang, Y.; Sun, B.; et~al. 2021{\natexlab{a}}.
\newblock Apan: Asynchronous propagation attention network for real-time
  temporal graph embedding.
\newblock In \emph{Proceedings of the 2021 International Conference on
  Management of Data}, 2628--2638.

\bibitem[{Wang et~al.(2021{\natexlab{b}})Wang, Cai, Liang, Ding, Wang, and
  Hooi}]{wang2021time}
Wang, Y.; Cai, Y.; Liang, Y.; Ding, H.; Wang, C.; and Hooi, B.
  2021{\natexlab{b}}.
\newblock Time-Aware Neighbor Sampling for Temporal Graph Networks.
\newblock \emph{arXiv:2112.09845}.

\bibitem[{Wang et~al.(2021{\natexlab{c}})Wang, Chang, Liu, Leskovec, and
  Li}]{wang2021inductive}
Wang, Y.; Chang, Y.-Y.; Liu, Y.; Leskovec, J.; and Li, P. 2021{\natexlab{c}}.
\newblock Inductive representation learning in temporal networks via causal
  anonymous walks.
\newblock \emph{arXiv:2101.05974}.

\bibitem[{Xu et~al.(2020)Xu, Ruan, Korpeoglu, Kumar, and
  Achan}]{xu2020inductive}
Xu, D.; Ruan, C.; Korpeoglu, E.; Kumar, S.; and Achan, K. 2020.
\newblock Inductive representation learning on temporal graphs.
\newblock \emph{arXiv:2002.07962}.

\bibitem[{Xu et~al.(2019)Xu, Chen, Liu, Chen, Weng, Hong, and
  Lin}]{xu2019topology}
Xu, K.; Chen, H.; Liu, S.; Chen, P.-Y.; Weng, T.~W.; Hong, M.; and Lin, X.
  2019.
\newblock Topology attack and defense for graph neural networks: An
  optimization perspective.
\newblock In \emph{International Joint Conference on Artificial Intelligence}.

\bibitem[{Zeager et~al.(2017)Zeager, Sridhar, Fogal, Adams, Brown, and
  Beling}]{zeager2017adversarial}
Zeager, M.~F.; Sridhar, A.; Fogal, N.; Adams, S.; Brown, D.~E.; and Beling,
  P.~A. 2017.
\newblock Adversarial learning in credit card fraud detection.
\newblock In \emph{2017 Systems and Information Engineering Design Symposium},
  112--116. IEEE.

\bibitem[{Zhang and Zitnik(2020)}]{zhang2020gnnguard}
Zhang, X.; and Zitnik, M. 2020.
\newblock Gnnguard: Defending graph neural networks against adversarial
  attacks.
\newblock \emph{Advances in Neural Information Processing Systems}, 33:
  9263--9275.

\bibitem[{Zhou et~al.(2022)Zhou, Zheng, Nisa, Ioannidis, Song, and
  Karypis}]{zhou2022tgl}
Zhou, H.; Zheng, D.; Nisa, I.; Ioannidis, V.; Song, X.; and Karypis, G. 2022.
\newblock Tgl: A general framework for temporal gnn training on billion-scale
  graphs.
\newblock \emph{arXiv:2203.14883}.

\bibitem[{Z{\"u}gner, Akbarnejad, and
  G{\"u}nnemann(2018)}]{zugner2018adversarial}
Z{\"u}gner, D.; Akbarnejad, A.; and G{\"u}nnemann, S. 2018.
\newblock Adversarial attacks on neural networks for graph data.
\newblock In \emph{Proceedings of the 24th ACM SIGKDD International Conference
  on Knowledge Discovery \& Data Mining}, 2847--2856.

\end{thebibliography}

\clearpage
\appendix
\section{Experiment Setup Details}
\subsection{Hardware Specification and Environment}

All experiments are performed on a server with NVIDIA RTX A6000 GPUs (48GB memory), 512GB of RAM, and two Intel Xeon Silver 4210R Processors.
Our models are implemented using PyTorch 1.12.1~\citep{paszke2019pytorch} and Deep Graph Library (DGL) 0.9.1~\citep{wang2019deep}.

\subsection{Details on Dataset}

We list the four datasets used in this paper.
\begin{itemize}
    \item \textbf{Wikipedia}\footnote{\url{http://snap.stanford.edu/jodie/wikipedia.csv}} is a user action dataset that captures interactions between users and Wikipedia pages.
    This dataset contains 157,474 attributed interactions recorded over one month, involving 8,227 users and 1,000 pages. 
    Each interaction represents a user editing a page, with the editing texts converted into LIWC-feature vectors.
    
    \item \textbf{MOOC}\footnote{\url{http://snap.stanford.edu/jodie/mooc.csv}} is a user action dataset containing actions performed by users on the MOOC platform. 
    It comprises interactions from 7,047 users with 97 items, resulting in 411,749 attributed interactions recorded over approximately one month. 
    These interactions represent the access behavior of students to online course units. 
    As edge features are not available in this dataset, we did not use them in our experiment.
    
    \item \textbf{UCI}\footnote{\url{https://snap.stanford.edu/data/CollegeMsg.html}} dataset contains message interactions among users within an online social media platform at the University of California, Irvine. 
    It consists of 59,835 message interactions involving 1,899 unique users.
    As edge features are not available in this dataset, we did not use them in our experiment.

    \item \textbf{Bitcoin-OTC}\footnote{\url{https://snap.stanford.edu/data/soc-sign-bitcoin-otc.html}} is a who-trusts-whom network of people who trade Bitcoin on the Bitcoin OTC platform. 
    Due to the anonymity of Bitcoin users, preserving records of user reputations is necessary to prevent transactions with fraudulent and high-risk users.
    The dataset encompasses 35,592 interactions among 5,881 nodes.
    As edge features are not available in this dataset, we did not use them in our experiment.
\end{itemize}

\subsubsection{Edge Features for Adversarial Edges}
Our proposed adversarial attack method can be extended to attributed dynamic graphs. 
For this purpose, when creating adversarial edges, we must generate not only their timing and two endpoints but also the edge features for them. The subsequent description relates to the constraints on generating edge features for adversarial edges.

For unnoticeability, we impose a distribution constraint on the edge features of the adversarial edges.
When generating an edge features $\tilde{\vece}_{\tilde{u}\tilde{v}}$ of an adversarial edge $\tilde{e}=(\tilde{u},\tilde{v},\tilde{t})$, where $t_{i}\leq\tilde{t}\leq t_{i+1}$, we ensure that the edge features $\tilde{\vece}_{\tilde{u}\tilde{v}}$ follows the edge feature distribution observed in the previous $W$ edges, i.e., $\tilde{\vece}_{\tilde{u}\tilde{v}}\sim P_{\vece}(\matE_{i,W})$.
Here, $P_{\vece}$ is the edge feature distribution, and $\matE_{i,W}$ is the set of edge features of the previous $W$ edges.
Specifically, we use Kernel Density Estimation (KDE) to approximate the edge feature distribution of the original edges. 
We use the Gaussian kernel with a bandwidth of 0.1. The edge features of the adversarial edges are drawn from the estimated edge feature distribution.

\begin{table}[h]
    \centering
    \scalebox{0.85}{
    \begin{tabular}{lcccccc}
        \toprule
        & $|\calV|$ & $|\calE|$ & $t_{max}$ & $|d_{e}|$ & Bipartite \\
        \midrule
        Wikipedia & 9,227 & 157,474 & $2.7\times 10^{6}$ & 172 & \ding{51} \\
        MOOC & 7,144 & 411,749 & $2.6\times 10^{6}$ & - & \ding{51} \\
        UCI & 1,899 & 59,835 & $1.6\times 10^{7}$ & - & - \\
        Bitcoin-OTC & 5,881 & 35,592 & $1.6\times 10^{8}$ & - & - \\
        \bottomrule
    \end{tabular}
    }
    \caption{Statistics of datasets used in our experiments.}
    \label{tab:dataset}
\end{table}






\subsection{Baselines}

\subsubsection{Attack Baselines.}
Since there are no existing adversarial attack methods on CTDGs to be used as baselines, we implemented five edge-perturbation baselines for comparison: 
\begin{enumerate}[leftmargin=*,label=(\arabic*),wide=0pt,noitemsep,topsep=0pt,parsep=0pt,partopsep=0pt]
    \item \textbf{\random}: It randomly links two nodes.
    \item \textbf{\pa}: It links nodes $u$ and $v$ with the lowest Preferential Attachment (PA) score~\citep{liben2003link} $|\calN(u)|\cdot|\calN(v)|$, the product of their neighbor counts. 
    \item \textbf{\jaccard}: It links nodes $u$ and $v$ with the lowest Jaccard coefficient~\citep{liben2003link}, calculated as $\frac{|\calN(u)\cap\calN(v)|}{|\calN(u)\cup\calN(v)|}$. As common neighbors cannot be identified among any pairs of nodes in bipartite graphs, we only apply this attack to unipartite graphs.
    \item \textbf{\degree}: Inspired by Structack~\citep{hussain2021structack}, it connects two nodes $u$ and $v$ by selecting the pair with the smallest sum of their neighbor counts, i.e., $|\calN(u)|+|\calN(v)|$.
    \item \textbf{\pagerank}: Similar to \degree, it connects two nodes $u$ and $v$ by selecting the pair with the smallest sum of their PageRank centralities, i.e., $PR(u)+PR(v)$.
\end{enumerate}
We consider these baselines under the constraints (C1-C4).
Note that, for baselines (2)-(5), we construct a plain graph using all the edges up to the point just before a perturbation is injected and subsequently compute their metrics.
In addition, for baselines (2)-(5), we employ our proposed Hungarian selection to create a pool for each batch. 

\subsubsection{Defense Baselines.}
As no established adversarial defense methods exist for CTDGs as baselines, we developed two baseline robust-training methods for comparison:
\begin{enumerate}[leftmargin=*,label=(\arabic*),wide=0pt,noitemsep,topsep=0pt,parsep=0pt,partopsep=0pt]
    \item \textbf{\tgnsvd}: Motivated by GCN-SVD~\citep{entezari2020all}, it builds a plain graph using all edges in the training set, and then computes a low-order approximation $\hat{\matA}$ of the adjacency matrix.
    It uses entries of $\hat{\matA}$ as weights in the loss function. 
    Specifically, for each edge $e=(u,v,t)$, $\hat{a}_{uv}$ is the weight of the edge's loss in Eq.~\eqref{eq:link_loss}.
    
    \item \textbf{\tgncosine}: Inspired by GNNGuard~\citep{zhang2020gnnguard}, it employs cosine similarity to filter edges.
    Specifically, it removes an edge $e=(u,v,t)$ if the cosine similarity between node embeddings $\vech_{u}(t)$ and $\vech_{v}(t)$ is below a predefined threshold $\tau_{cosine}$.
\end{enumerate}

\subsection{Implementation Details}
For all our experiments, we use a constant learning rate of 0.0001 while training all victim models, surrogate models for attack methods, and defense models. 
In particular, we employ the Adam optimizer~\citep{kingma2014adam} with a batch size of 600 and train the models for a total of 100 epochs (50 epochs for DySAT).
After training 50 epochs (25 epochs for DySAT), we apply early stopping if validation MRR does not improve for 10 epochs.


\subsubsection{Hyperparameters for the Victim Models}
The implementation of the victim models (TGN~\citep{rossi2020temporal}, JODIE~\citep{kumar2019predicting}, TGAT~\citep{xu2020inductive}, and DySAT~\citep{sankar2018dynamic}) follows the temporal graph learning framework, called TGL\footnote{The TGL framework can be found at \url{https://github.com/amazon-research/tgl}}~\citep{zhou2022tgl}.
TGL’s implementation could achieve a better overall score than the victim models' original implementation.
Hence, we employ TGL's default settings (e.g., learning rate 0.0001, batch size 600, hidden dimension 100, etc) for the victim models across all datasets.
The hyperparameters we used are summarized in Table~\ref{tab:hyper:victim}.

\begin{table}[h]
    \setlength{\tabcolsep}{4.5pt}
    \setlength\aboverulesep{0.5pt}
    \centering
    \small
    \begin{tabular}{lcccc}
        \toprule
        & \TGN & \JODIE & \TGAT & \DySAT \\
        \midrule
        Batch size & 600 & 600 & 600 & 600 \\ 
        Learning rate & 0.0001 & 0.0001 & 0.0001 & 0.0001 \\
        Node embedding dim & 100 & 100 & 100 & 100 \\
        Time embedding dim & 100 & 100 & 100 & 100 \\
        \# layers & 1 & - & 2 & 2 \\
        \# attention heads & 2 & - & 2 & 2\\
        \# sampled neighbors & 10 & - & 10 & 10 \\ 
        Dropout & 0.2 & 0.1 & 0.1 & 0.1 \\
        Duration* & - & - & - & 10,000 \\
        \bottomrule
        \multicolumn{5}{l}{*Length in time of each snapshot for discrete-time TGNNs.}\\
        \multicolumn{5}{l}{*For Bitcoin-OTC dataset, we use 100,000.}\\

    \end{tabular}
    \normalsize
    \caption{Hyperparameters for victim models.}
    \label{tab:hyper:victim}
\end{table}

\subsubsection{Hyperparameters for the Attack Models}

For all attack models, encompassing both \attack and baselines, we use a time window $W$ of 1200 to maintain unnoticeability. 
For \attack, we begin by training the surrogate model (TGN), and its hyperparameters mirror those of the victim TGN model by default.
In Appendix C, we validate the effectiveness of our attack when the attacker lacks knowledge about the internals (e.g., embedding size) of the victim model.

\subsubsection{Hyperparameters for the Defense Models}
All defense models utilize TGN as the backbone model. 
To ensure fairness in comparison, we adopt the same hyperparameters as the victim TGN model.
All additional hyperparameters are tuned using validation MRR under \random and \attack attacks with a perturbation rate of 0.3.
Empirically, the hyperparameters that perform well in both attacks are quite similar, hence we use the same hyperparameters for both scenarios.
Regarding \tgnsvd, an additional hyperparameter is introduced: the reduced rank of the perturbed graphs, which is tuned within a range of 10 to 100, with increments of 10.
For \tgncosine, we tune a threshold value $\tau_{cosine}$ from $\{0.01, 0.02, 0.03, 0.05, 0.1, 0.2\}$
In the case of \defense, we tune two thresholds, $\tau_{s}$ and $\tau_{e}$, within a range of 0 to 1, with steps of 0.1. 
Additionally, we search for a hyperparameter $\lambda$, which controls the level of temporal smoothness, from $\{0.01, 0.05, 0.1, 0.5, 1\}$.
Table~\ref{tab:hyper:defense:proposed} provides our chosen hyperparameters.

\begin{table}[h]
    \setlength{\tabcolsep}{5pt}
    \setlength\aboverulesep{0.5pt}
    \centering
    \small
    \begin{tabular}{lcccc}
        \toprule
        & Wikipedia & MOOC & UCI & Bitcoin-OTC \\
        \midrule
        Rank & 100 & 80 & 30 & 100 \\ 
        $\tau_{cosine}$ & 0.1 & 0.2 & 0.1 & 0.1 \\
        $\tau_{s}$ & 0.6 & 0.6 & 0.4 & 0.0 \\
        $\tau_{e}$ & 0.9 & 0.8 & 0.6 & 0.1 \\
        $\lambda$ & 0.05 & 0.1 & 0.05 & 0.05 \\
        \bottomrule
    \end{tabular}
    \normalsize
    \caption{Hyperparameters for defense models under \random and \attack attacks.}
    \label{tab:hyper:defense:proposed}
\end{table}

\section{Proofs of the Lemmas}
Here, we provide the proofs of the lemmas~\ref{lemma:attack} and~\ref{lemma:defense}.
To focus on the time complexities of the proposed methods, we denote the time complexity of TGN for the inference of a single node's embedding as $C_{tgn}$ and the time complexity of the edge scorer as $C_{clf}$ for simplicity.
Let $d$ be the dimension of node embeddings, $S$ be the number of sampled neighbors, and $L$ be the number of layers.
The time complexity of TGN is dominated by the attention layer, hence $C_{tgn}=O\big(d\cdot(d+S)\cdot L\big)$~\citep{wang2021time}.
When using a 2-layer MLP as the edge scorer, its time complexity is $C_{clf}=O(d^{2})$.

\setcounter{lemma}{0}
\begin{lemma}[Time complexity of \attack]
    The time complexity of \attack is $O\big((C_{tgn}+|W|\cdot C_{clf}+p\cdot|W|^{2})\cdot|\calE|\big)$.
\end{lemma}
\begin{proof}
For each batch $b$, computing embeddings of all nodes in $\calV_{b-1}$ takes $O(|B|\cdot C_{tgn})$ time.
Next, computing edge scores for all node pairs within $\calV_{b-1}$ takes $O(|B|^{2}\cdot C_{clf})$ time.
The time complexity of the Hungarian algorithm is $O(|B|^{3})$~\citep{jonker1988shortest}, and we repeat it $N=\lceil\frac{K}{E_{max}(\calV_{b-1})}\rceil$ times, i.e., $O(|B|^{3}\cdot N)=O(|B|^{3}\cdot\frac{p\cdot|B|}{|B|})=O(p\cdot|B|^{3})$.
The number of batches is $\frac{|\calE|}{|B|}$ and $|B|=\frac{|W|}{2}$.
The overall time complexity of \attack is $O\big((C_{tgn}+|W|\cdot C_{clf}+p\cdot|W|^{2})\cdot|\calE|\big)$.
\end{proof}

\begin{lemma}[Time complexity of \defense]
    The time complexity of \defense is $O\big((C_{tgn}+C_{clf}+d)\cdot|\calE|\big)$
\end{lemma}
\begin{proof}
For each edge $e=(u,v,t)$, computing embeddings of two nodes $u$ and $v$ takes $O(C_{tgn})$ time for each.
Next, computing the edge score of the edge $e$ takes $O(C_{clf})$ time.
After edge filtering, for each remaining edge $e_{c}=(u_{c},v_{c},t_{c})\in\calE_{c}$, we compute embeddings of two nodes $u_{c}$ and $v_{c}$, and the edge score between them. 
It takes $O(C_{tgn}+C_{clf})$ time.
For enforcing temporal robustness, computing the cosine similarity between the current and previous embedding of $u_{c}$ and $v_{c}$ takes $O(d)$ time, where $d$ is the dimension of node embeddings.
The number of remaining edges can be at most the total number of edges, i.e., $O(|\calE_{c}|)=O(|\calE|)$, and thus the overall time complexity of \defense is $O\big((C_{tgn}+C_{clf}+d)\cdot|\calE|\big)$.
\end{proof}

\section{Additional Experimental Results}
\subsection{Attack Performance}

Table~\ref{tab:attack_hit} provides a summary of the link prediction performance, measured by Hit@10 (MRR results can be found in the main paper), for all victim models across various attack methods at a perturbation rate of 0.3.
Across all datasets and victim models, our proposed attack consistently outperforms the baselines (lower values indicate better attack performance).
Overall, \attack exhibits the most notable performance degradation, averaging -13.3\%. 

\subsubsection{When the Victim Model's Internals are Unknown.}
In Figure~\ref{fig:dimension}, we illustrate the attacker's ability to conduct a successful attack on the victim model, even without knowing the exact embedding size used by the victim model. 
We assessed the attack's efficacy by using the embedding size of the surrogate model that varied from 0.5$\times$ to 4$\times$ that of the victim model. 
Notably, the attack consistently performed well regardless of the surrogate model's embedding size, indicating the applicability of our approach in situations where the attacker lacks knowledge of the victim model.

\begin{figure}
    \centering
    \begin{subfigure}{0.49\columnwidth}
        \includegraphics[width=\columnwidth]{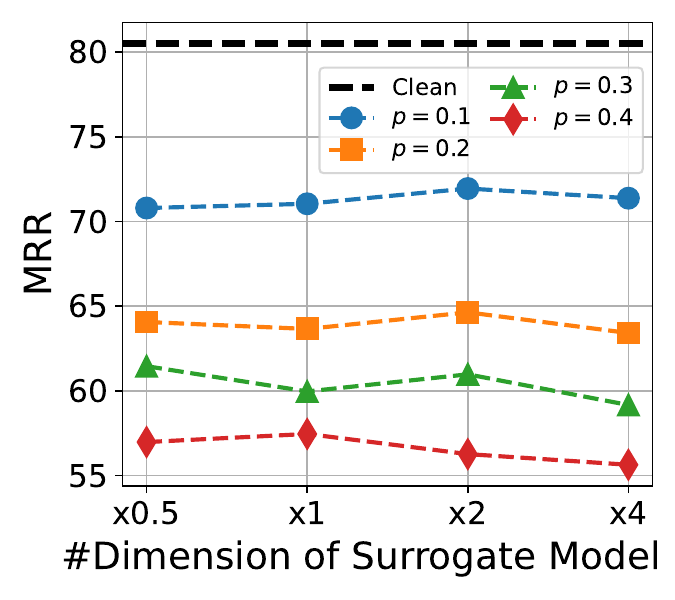}
    \end{subfigure}
    \begin{subfigure}{0.49\columnwidth}
        \includegraphics[width=\columnwidth]{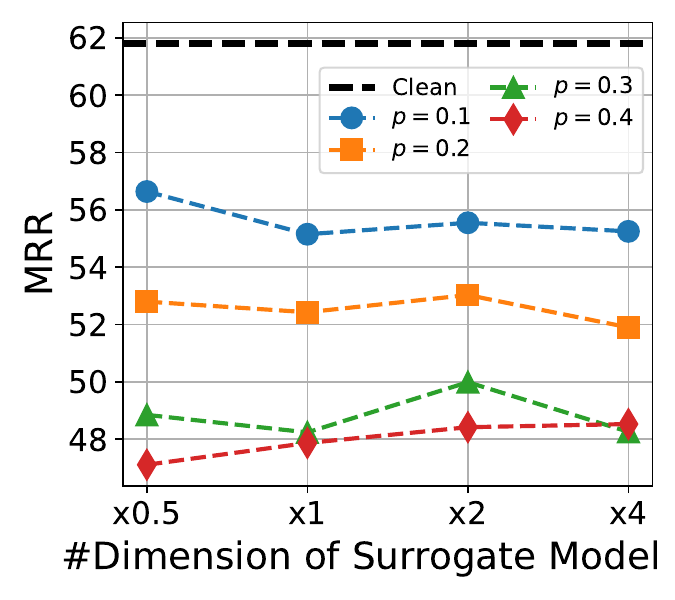}
    \end{subfigure}
    \caption{Link prediction performance of TGN on Wikipedia (left) and MOOC (right) as the embedding size of the surrogate model changes 0.5$\times$ to 4$\times$ that of the victim model. \attack consistently performed well regardless of the surrogate model's embedding size.}
    \label{fig:dimension}
\end{figure}

\begin{table*}[!t]
    \setlength{\tabcolsep}{5pt}
    \setlength\aboverulesep{0.5pt}
    \centering
    \scalebox{0.9}{
    \begin{tabular}{llccccccc}
        \toprule
        Dataset & Victim & Clean & \random & \pa & \jaccard\!* & \degree & \pagerank & \textbf{\attack} \\
        \midrule
        \multirow{4}{*}{Wikipedia}  & \TGN      & 91.3 $\pm$ 0.2 & 81.7 $\pm$ 0.7 & \ul{80.6 $\pm$ 0.9} & - & 82.1 $\pm$ 0.7 & 82.1 $\pm$ 1.1 & \tb{76.1 $\pm$ 1.6} \\
                                    & \JODIE    & 80.6 $\pm$ 0.9 & 59.2 $\pm$ 4.1 & 60.2 $\pm$ 3.3 & - & \ul{58.5 $\pm$ 2.4} & 60.7 $\pm$ 3.4 & \tb{57.9 $\pm$ 4.4} \\
                                    & \TGAT     & 82.4 $\pm$ 0.3 & 66.2 $\pm$ 0.7 & \tb{64.0 $\pm$ 1.4} & - & 65.9 $\pm$ 0.6 & 65.9 $\pm$ 0.6 & \ul{64.9 $\pm$ 0.3} \\
                                    & \DySAT    & 80.4 $\pm$ 0.1 & 78.7 $\pm$ 0.1 & 77.9 $\pm$ 0.1 & - & \ul{77.8 $\pm$ 0.2} & \tb{77.8 $\pm$ 0.1} & 78.2 $\pm$ 0.1 \\
        \midrule
        \multirow{4}{*}{MOOC}   & \TGN          & 66.7 $\pm$ 2.1 & 60.9 $\pm$ 1.7 & 61.1 $\pm$ 1.7 & - & \ul{55.5 $\pm$ 1.5} & 58.3 $\pm$ 1.0 & \tb{53.3 $\pm$ 1.5} \\
                                & \JODIE        & 41.5 $\pm$ 2.4 & \ul{27.1 $\pm$ 0.5} & 32.8 $\pm$ 4.1 & - & 29.3 $\pm$ 1.9 & 30.1 $\pm$ 1.6 & \tb{25.9 $\pm$ 1.9} \\
                                & \TGAT         & 23.3 $\pm$ 0.2 & 20.3 $\pm$ 0.4 & 19.7 $\pm$ 0.1 & - & 21.8 $\pm$ 0.5 & \ul{19.3 $\pm$ 0.3} & \tb{18.6 $\pm$ 0.4} \\
                                & \DySAT        & 28.8 $\pm$ 0.1 & 23.9 $\pm$ 0.2 & \ul{23.4 $\pm$ 0.4} & - & 23.6 $\pm$ 0.3 & \tb{23.4 $\pm$ 0.2} & 23.7 $\pm$ 0.2 \\
        \midrule
        \multirow{4}{*}{UCI}    & \TGN      & 60.0 $\pm$ 1.2 & 58.4 $\pm$ 0.7 & 63.2 $\pm$ 1.9 & 59.4 $\pm$ 1.5 & \ul{58.2 $\pm$ 1.6} & 58.7 $\pm$ 1.6 & \tb{57.8 $\pm$ 1.0} \\
                                & \JODIE    & 68.3 $\pm$ 0.6 & 63.8 $\pm$ 1.6 & \ul{63.4 $\pm$ 0.8} & 63.7 $\pm$ 1.0 & \tb{62.6 $\pm$ 1.2} & 63.5 $\pm$ 0.8 & 63.6 $\pm$ 1.2 \\
                                & \TGAT     & 27.5 $\pm$ 0.7 & 26.6 $\pm$ 0.8 & \tb{24.4 $\pm$ 0.4} & 26.1 $\pm$ 0.4 & 25.0 $\pm$ 0.6 & 27.9 $\pm$ 0.9 & \ul{24.5 $\pm$ 0.7} \\
                                & \DySAT    & 98.9 $\pm$ 1.6 & 98.3 $\pm$ 1.1 & 98.5 $\pm$ 0.5 & \tb{97.8 $\pm$ 1.1} & \ul{98.2 $\pm$ 0.4} & 98.4 $\pm$ 0.3 & 98.6 $\pm$ 0.6 \\
        \midrule
        \multirow{4}{*}{Bitcoin}    & \TGN      & \ul{66.9 $\pm$ 0.8} & 67.8 $\pm$ 1.2 & 68.0 $\pm$ 0.7 & 67.6 $\pm$ 1.0 & 67.0 $\pm$ 0.6 & 68.1 $\pm$ 0.9 & \tb{66.6 $\pm$ 1.8} \\
                                    & \JODIE    & \ul{64.8 $\pm$ 0.7} & 65.9 $\pm$ 0.6 & 65.3 $\pm$ 0.5 & 66.1 $\pm$ 0.6 & 65.5 $\pm$ 0.6 & \tb{64.8 $\pm$ 0.6} & 65.1 $\pm$ 0.3 \\
                                    & \TGAT     & 41.3 $\pm$ 0.4 & 33.6 $\pm$ 0.7 & 31.5 $\pm$ 0.8 & 33.9 $\pm$ 0.8 & \ul{29.8 $\pm$ 1.3} & \tb{28.9 $\pm$ 1.6} & 30.6 $\pm$ 0.6 \\
                                    & \DySAT    & 99.2 $\pm$ 1.0 & \tb{98.7 $\pm$ 1.0} & 99.4 $\pm$ 0.4 & 99.4 $\pm$ 0.4 & 99.3 $\pm$ 0.6 & 99.3 $\pm$ 0.6 & \ul{99.0 $\pm$ 0.5} \\
        \midrule    
        \multicolumn{2}{c}{A.P.D.$\downarrow$} & - & -10.1\% & -10.0\% & -3.6\% & \ul{-11.2\%} & -10.6\% & \tb{-13.3\%} \\         
        \bottomrule             
        \multicolumn{9}{l}{ *\jaccard cannot be applied to bipartite graphs.}
    \end{tabular}}
    \caption{Hit@10 at a perturbation rate $p=0.3$ (MRR results can be found in the main paper). 
    For each victim model, the best and second-best performances are highlighted in boldface and underlined, respectively. 
    A.P.D. stands for average performance degradation.
    As we choose the model's best parameters using validation MRR, Hit@10 might be somewhat unstable.
    However, across all datasets and victim models, \attack consistently outperforms the baselines.}
    \label{tab:attack_hit}
\end{table*}

\begin{table*}[!t]
    \setlength{\tabcolsep}{3.5pt}
    \setlength\aboverulesep{0.5pt}
    \centering
    \scalebox{0.9}{
    \begin{tabular}{lccccccccccc}
        \toprule
        \multirow{2}{*}[-1mm]{Model} & \multicolumn{3}{c}{Wikipedia (Clean: 91.3 $\pm$ 0.2)} & \multicolumn{3}{c}{MOOC (Clean: 66.7 $\pm$ 2.1)} & \multicolumn{3}{c}{UCI (Clean: 60.0 $\pm$ 1.2)} & \multirow{2}{*}[-1mm]{A.P.G.$\uparrow$} \\
        \cmidrule(lr){2-4} \cmidrule(lr){5-7} \cmidrule(lr){8-10} 
        & $p=0.1$ & $p=0.2$ & $p=0.3$ & $p=0.1$ & $p=0.2$ & $p=0.3$ & $p=0.1$ & $p=0.2$ & $p=0.3$ \\
        \midrule

        \TGN                    & 88.1 $\pm$ 0.5 & 84.4 $\pm$ 0.7 & 81.7 $\pm$ 0.7 & 63.4 $\pm$ 1.2 & 63.0 $\pm$ 0.8 & 60.9 $\pm$ 1.7 & 62.8 $\pm$ 1.6 & 60.7 $\pm$ 3.0 & 58.4 $\pm$ 0.7 & - \\
        \tgnsvd                 & 85.0 $\pm$ 0.7 & 81.4 $\pm$ 1.2 & 78.9 $\pm$ 0.9 & \tb{63.8 $\pm$ 0.5} & 63.0 $\pm$ 1.2 & 60.3 $\pm$ 1.0 & 65.9 $\pm$ 1.3 & 66.5 $\pm$ 0.5 & \ul{68.1 $\pm$ 1.4} & +2.2\% \\
        \tgncosine              & 87.6 $\pm$ 0.2 & 85.7 $\pm$ 0.1 & 83.5 $\pm$ 0.6 & 59.0 $\pm$ 2.6 & 57.7 $\pm$ 3.8 & 58.9 $\pm$ 3.0 & 65.7 $\pm$ 1.1 & 66.0 $\pm$ 1.2 & 67.8 $\pm$ 0.5 & +1.6\% \\
        \midrule
        \textbf{\defensefilter} & \ul{88.1 $\pm$ 0.3} & \tb{88.3 $\pm$ 0.4} & \ul{87.9 $\pm$ 0.1} & \ul{63.8 $\pm$ 1.0} & \tb{63.2 $\pm$ 1.3} & \ul{62.5 $\pm$ 0.8} & \tb{66.1 $\pm$ 0.9} & \tb{67.4 $\pm$ 1.2} & 67.6 $\pm$ 0.9 & \ul{+5.3\%} \\
        \textbf{\defense}       & \tb{88.1 $\pm$ 0.3} & \ul{88.0 $\pm$ 0.3} & \tb{88.0 $\pm$ 0.2} & 63.6 $\pm$ 1.4 & \ul{63.1 $\pm$ 1.1} & \tb{62.9 $\pm$ 0.7} & \ul{65.9 $\pm$ 0.5} & \ul{67.4 $\pm$ 1.3} & \tb{68.4 $\pm$ 1.3} & \tb{+5.4\%} \\
        
        \bottomrule             
    \end{tabular}}
    \caption{Hit@10 under \random attack (MRR results can be found in the main paper). A.P.G. stands for average performance gain over the na\"ive TGN.
    \defense and \defensefilter consistently enhance the link prediction performance across various attack methods, datasets, and perturbation rates.
    }
    \label{tab:defense:random_hit}
\end{table*}

\begin{table*}[!t]
    \setlength{\tabcolsep}{3.5pt}
    \setlength\aboverulesep{0.5pt}
    \centering
    \scalebox{0.9}{
    \begin{tabular}{lccccccccccc}
        \toprule
        \multirow{2}{*}[-1mm]{Model} & \multicolumn{3}{c}{Wikipedia (Clean: 91.3 $\pm$ 0.2)} & \multicolumn{3}{c}{MOOC (Clean: 66.7 $\pm$ 2.1)} & \multicolumn{3}{c}{UCI (Clean: 60.0 $\pm$ 1.2)} & \multirow{2}{*}[-1mm]{A.P.G.$\uparrow$} \\
        \cmidrule(lr){2-4} \cmidrule(lr){5-7} \cmidrule(lr){8-10} 
        & $p=0.1$ & $p=0.2$ & $p=0.3$ & $p=0.1$ & $p=0.2$ & $p=0.3$ & $p=0.1$ & $p=0.2$ & $p=0.3$ \\
        \midrule

        \TGN                    & 86.3 $\pm$ 1.4 & 80.0 $\pm$ 2.8 & 76.1 $\pm$ 1.6 & 60.7 $\pm$ 0.7 & 57.7 $\pm$ 0.4 & 53.3 $\pm$ 1.5 & 60.8 $\pm$ 3.4 & 57.4 $\pm$ 0.6 & 57.8 $\pm$ 1.0 & - \\
        \tgnsvd                 & 85.9 $\pm$ 0.6 & 79.8 $\pm$ 1.1 & 75.6 $\pm$ 2.2 & 61.1 $\pm$ 1.1 & 55.8 $\pm$ 2.3 & 55.0 $\pm$ 2.1 & \tb{68.6 $\pm$ 1.1} & \tb{68.4 $\pm$ 0.3} & \tb{71.0 $\pm$ 0.3} & +6.0\% \\
        \tgncosine              & 85.8 $\pm$ 0.7 & 78.8 $\pm$ 1.9 & 76.1 $\pm$ 2.0 & 58.8 $\pm$ 2.8 & 52.9 $\pm$ 1.6 & 50.2 $\pm$ 2.0 & 65.3 $\pm$ 0.5 & 66.6 $\pm$ 1.4 & 64.8 $\pm$ 2.6 & +1.8\% \\
        \midrule
        \textbf{\defensefilter} & \ul{88.0 $\pm$ 0.4} & \ul{87.7 $\pm$ 0.4} & \ul{86.7 $\pm$ 0.4} & \ul{63.1 $\pm$ 0.5} & \ul{60.2 $\pm$ 1.0} & \tb{57.5 $\pm$ 2.9} & 65.8 $\pm$ 0.9 & \ul{68.2 $\pm$ 1.3} & 67.4 $\pm$ 1.5 & \ul{+9.5\%} \\
        \textbf{\defense}       & \tb{88.0 $\pm$ 0.3} & \tb{87.8 $\pm$ 0.1} & \tb{87.2 $\pm$ 0.3} & \tb{63.3 $\pm$ 0.7} & \tb{60.8 $\pm$ 1.1} & \ul{56.6 $\pm$ 2.1} & \ul{66.5 $\pm$ 0.6} & 67.0 $\pm$ 1.3 & \ul{69.5 $\pm$ 0.7} & \tb{+9.8\%} \\
        
        \bottomrule             
    \end{tabular}}
    \caption{Hit@10 under \attack attack (MRR results can be found in the main paper). A.P.G. stands for average performance gain over the na\"ive TGN.
    \defense and \defensefilter consistently enhance the link prediction performance across various attack methods, datasets, and perturbation rates.
    }
    \label{tab:defense:proposed_hit}
\end{table*}

\subsection{Defense Performance}
Table~\ref{tab:defense:random_hit} and~\ref{tab:defense:proposed_hit} present the performance of the defense models in terms of Hit@10 (MRR results can be found in the main paper) against the \random and \attack attacks, respectively. 
Importantly, \defense and \defensefilter consistently enhance the link prediction performance across various attack methods, datasets, and perturbation rates. 
In Table~\ref{tab:defense:bitcoin:mrr} and~\ref{tab:defense:bitcoin:hit}, we provide additional experimental results for defense models on the Bitcoin-OTC dataset.
In the Bitcoin-OTC dataset, the performance gap between the clean graph and the attacked graphs is not significant; therefore, the improvement in the defense model's performance is not noticeable. 
Nonetheless, there was still an average performance increase of 1.9\%.

\begin{table*}[!t]
    \setlength{\tabcolsep}{3.5pt}
    \setlength\aboverulesep{0.5pt}
    \centering
    \scalebox{0.9}{
    \begin{tabular}{lcccccccc}
        \toprule
        \multirow{2}{*}[-1mm]{Model} & \multicolumn{3}{c}{Under \random Attack} & \multicolumn{3}{c}{Under \attack Attack} & \multirow{2}{*}[-1mm]{A.P.G.$\uparrow$} \\
        \cmidrule(lr){2-4} \cmidrule(lr){5-7} 
        & $p=0.1$ & $p=0.2$ & $p=0.3$ & $p=0.1$ & $p=0.2$ & $p=0.3$ \\
        \midrule

        \TGN                    & 47.6 $\pm$ 0.4 & 47.2 $\pm$ 1.0 & 48.0 $\pm$ 1.3 & 47.9 $\pm$ 0.4 & 46.7 $\pm$ 1.1 & 45.2 $\pm$ 1.4 & - \\
        \tgnsvd                 & 46.8 $\pm$ 0.6 & 47.0 $\pm$ 0.6 & \tb{48.1 $\pm$ 0.4} & 48.0 $\pm$ 0.3 & \ul{48.2 $\pm$ 0.7} & \tb{48.2 $\pm$ 0.7} & +1.4\% \\
        \tgncosine              & 41.5 $\pm$ 0.5 & 42.9 $\pm$ 0.6 & 44.2 $\pm$ 1.0 & 42.4 $\pm$ 1.1 & 43.1 $\pm$ 0.9 & 43.2 $\pm$ 0.5 & -8.9\% \\
        \midrule
        \textbf{\defensefilter} & \ul{48.0 $\pm$ 0.7} & \ul{48.0 $\pm$ 0.9} & \ul{48.1 $\pm$ 0.8} & \tb{48.5 $\pm$ 0.6} & 47.5 $\pm$ 1.4 & 46.8 $\pm$ 1.5 & \ul{+1.5\%} \\
        \textbf{\defense}       & \tb{48.0 $\pm$ 0.5} & \tb{48.0 $\pm$ 0.6} & 48.0 $\pm$ 0.7 & \ul{48.4 $\pm$ 0.7} & \tb{48.4 $\pm$ 0.7} & \ul{47.2 $\pm$ 1.4} & \tb{+1.9\%} \\
        
        \bottomrule             
    \end{tabular}}
    \caption{MRR under \random and \attack attacks on the Bitcoin-OTC dataset. A.P.G. stands for average performance gain over the na\"ive TGN. MRR of the clean dataset is $48.3 \pm 0.7$.
    In this dataset, the performance gap between the clean graph and the attacked graphs is not significant; therefore, the improvement in the defense model's performance is not noticeable. 
    Nonetheless, there was still an average performance increase of 1.9\%.}
    \label{tab:defense:bitcoin:mrr}
\end{table*}

\begin{table*}[!t]
    \setlength{\tabcolsep}{3.5pt}
    \setlength\aboverulesep{0.5pt}
    \centering
    \scalebox{0.9}{
    \begin{tabular}{lccccccccccc}
        \toprule
        \multirow{2}{*}[-1mm]{Model} & \multicolumn{3}{c}{Under \random Attack} & \multicolumn{3}{c}{Under \attack Attack} & \multirow{2}{*}[-1mm]{A.P.G.$\uparrow$} \\
        \cmidrule(lr){2-4} \cmidrule(lr){5-7} 
        & $p=0.1$ & $p=0.2$ & $p=0.3$ & $p=0.1$ & $p=0.2$ & $p=0.3$ \\
        \midrule

        \TGN                    & \ul{67.5 $\pm$ 0.4} & 66.7 $\pm$ 1.0 & \tb{67.8 $\pm$ 1.2} & 67.6 $\pm$ 0.8 & 66.1 $\pm$ 1.0 & 66.6 $\pm$ 1.8 & - \\
        \tgnsvd                 & 67.0 $\pm$ 0.3 & 66.8 $\pm$ 0.7 & \ul{67.4 $\pm$ 0.4} & 67.8 $\pm$ 0.5 & \tb{67.3 $\pm$ 0.7} & \tb{69.7 $\pm$ 0.3} & \tb{+0.9\%} \\
        \tgncosine              & 60.0 $\pm$ 0.9 & 62.1 $\pm$ 0.9 & 64.2 $\pm$ 1.1 & 60.1 $\pm$ 0.7 & 61.3 $\pm$ 1.4 & 62.5 $\pm$ 1.3 & -8.0\% \\
        \midrule
        \textbf{\defensefilter} & \tb{67.7 $\pm$ 0.6} & \tb{67.4 $\pm$ 1.0} & 66.6 $\pm$ 0.7 & \ul{68.0 $\pm$ 1.2} & 66.4 $\pm$ 1.7 & \ul{68.0 $\pm$ 1.5} & +0.5\% \\
        \textbf{\defense}       & 67.5 $\pm$ 0.8 & \ul{67.0 $\pm$ 1.1} & 66.5 $\pm$ 0.7 & \tb{68.4 $\pm$ 0.5} & \ul{67.3 $\pm$ 0.9} & 67.8 $\pm$ 1.2 & \ul{+0.6\%} \\
        \bottomrule             
    \end{tabular}}
    \caption{Hit@10 under \random and \attack attacks on the Bitcoin-OTC dataset. A.P.G. stands for average performance gain over the na\"ive TGN. Hit@10 of the clean dataset is $66.9 \pm 0.8$.
    Since we optimize the defense model with validation MRR, Hit@10 measurements might be unstable.
    In this case, the baseline \tgnsvd outperforms \defense.
    However, the difference is not very significant, and our proposed defense method ranks second.
    }
    \label{tab:defense:bitcoin:hit}
\end{table*}


\end{document}